\documentclass{article}

\def\neurips{0}

\ifnum\neurips=1
    
         \usepackage[preprint]{neurips_2021}
    
\else 
\fi

\ifnum\neurips=1
    \usepackage[utf8]{inputenc} 
    \usepackage[T1]{fontenc}    
    \usepackage{hyperref}       
    \usepackage{url}            
    \usepackage{booktabs}       
    \usepackage{amsfonts}       
    \usepackage{nicefrac}       
    \usepackage{microtype}      
    \usepackage{xcolor}         
\else 
\fi

\usepackage[pdftex]{graphicx}
\usepackage[utf8]{inputenc}
\usepackage{amsmath,fullpage}
\usepackage{amssymb}
\usepackage{amsthm}
\usepackage{algorithm}
\usepackage[noend]{algpseudocode}
\usepackage{algorithmicx}
\usepackage[dvipsnames]{xcolor}
\usepackage{comment}
\usepackage{xspace}
\usepackage{bm}
\usepackage{thmtools}
\usepackage{thm-restate}
\usepackage[colorlinks,citecolor=Green,urlcolor=blue,bookmarks=false,hypertexnames=true]{hyperref}
\usepackage{cleveref}

\newtheorem{theorem}{Theorem}[section]

\newtheorem{lemma}[theorem]{Lemma}
\newtheorem{claim}[theorem]{Claim}
\newtheorem{corollary}[theorem]{Corollary}
\newtheorem{definition}[theorem]{Definition}

\title{Multiclass versus Binary Differentially Private PAC Learning}

\ifnum\neurips=0
    \author{Mark Bun\thanks{Boston University, \url{mbun@bu.edu}} \and Marco Gaboardi\thanks{Boston University, \url{gaboardi@bu.edu}} \and Satchit Sivakumar\thanks{Boston University, \url{satchit@bu.edu}}}
\else 
   \author{}
\fi
\date{\today}

\begin{document}
\maketitle

\begin{abstract}
We show a generic reduction from multiclass differentially private PAC learning to binary private PAC learning. We apply this transformation to a recently proposed binary private PAC learner to obtain a private multiclass learner with sample complexity that has a polynomial dependence on the multiclass Littlestone dimension and a poly-logarithmic dependence on the number of classes. This yields an exponential improvement in the dependence on both parameters over learners from previous work. Our proof extends the notion of $\Psi$-dimension defined in work of Ben-David et al. \cite{Ben} to the online setting and explores its general properties. 
\end{abstract}

\tableofcontents

\ifnum\neurips=0
\newpage
\fi 

\section{Introduction}

Machine learning and data analytics are increasingly deployed on sensitive information about individuals. Differential privacy~\cite{DMNS06} gives a mathematically rigorous way to enable such analyses while guaranteeing the privacy of individual information. The model of \emph{differentially private PAC learning}~\cite{KLNRS} captures  binary classification for sensitive data, providing a simple and broadly applicable abstraction for many machine learning procedures. Private PAC learning is now reasonably well-understood, with a host of general algorithmic techniques, lower bounds, and results for specific fundamental concept classes~~\cite{bnsv,feldman,bns13,BNS19,almm,KLMNS19,hspace2,hspace1}.

Beyond binary classification, many problems in machine learning are better modeled as \emph{multiclass learning} problems. Here, given a training set of examples from domain $\mathcal{X}$ with labels from $[k] = \{0, 1, \dots, k\}$, the goal is to learn a function $h : \mathcal{X} \to [k]$ that approximately labels the data and generalizes to the underlying population from which it was drawn. Much less is presently known about differentially private multiclass learnability than is known about private binary classification, though it appears that many specific tools and techniques can be adapted one at a time. In this work, we ask: \emph{Can we generically relate multiclass to binary learning so as to automatically transfer results from the binary setting to the multiclass setting?}

To illustrate, there is a simple reduction from a given multiclass learning problem to a sequence of binary classification problems. (This reduction was described by Ben-David et al.~\cite{Ben} for non-private learning, but works just as well in the private setting.) Intuitively, one can learn a multi-valued label one bit at a time. That is, to learn an unknown function $f : \mathcal{X} \to [k]$, it suffices to learn the $\lceil \log_2(k + 1) \rceil$ binary functions $f_i : \mathcal{X} \to [k]$, where each $f_i$ is the $i^{th}$ bit of $f$. 

\begin{theorem}[Informal] \label{lem:generic-informal}
Let $H$ be a concept class consisting of $[k]$-valued functions. If all of the binary classes $H|_i = \{f_i : f \in H\}$ are privately learnable, then $H$ is privately learnable.
\end{theorem}

Beyond its obvious use for enabling the use of tools for binary private PAC learning on the classes $H|_i$, we show that Theorem~\ref{lem:generic-informal} has strong implications for relating the private learnability of $H$ to the combinatorial properties of $H$ itself. Our main application of this reductive perspective is an improved sample complexity upper bound for private multiclass learning in terms of online learnability.

\subsection{Online vs. Private Learnability}

A recent line of work has revealed an intimate connection between differentially private learnability and learnability in Littlestone's mistake-bound model of online learning \cite{littlestone}. For binary classes, the latter is tightly captured by a combinatorial parameter called the Littlestone dimension; a class $H$ is online learnable with mistake bound at most $d$ if and only if its Littlestone dimension is at most $d$. The Littlestone dimension also qualitatively characterizes private learnability. If a class $H$ has Littlestone dimension $d$,  then every private PAC learner for $H$ requires at least $\Omega(\log^* d)$ samples~\cite{almm}. Meanwhile, Bun et al.~\cite{bun} showed that $H$ is privately learnable using $2^{O(2^d)}$ samples, and Ghazi et al.~\cite{ghazi} gave an improved algorithm using $\tilde{O}(d^6)$ samples. (Moreover, while quantitatively far apart, both the upper and lower bound are tight up to polynomial factors as functions of the Littlestone dimension alone~\cite{KLMNS19}.)

Jung et al.~\cite{jkt} recently extended this connection from binary to multiclass learnability. They gave upper and lower bounds on the sample complexity of private multiclass learnability in terms of the \emph{multiclass Littlestone dimension}~\cite{daniely1}. Specifically, they showed that if a multi-valued class $H$ has multiclass Littlestone dimension $d$, then it is privately learnable using $2^{O(k^d)}$ samples and that every private learner requires $\Omega(\log^* d)$ samples.

Jung et al.'s upper bound~\cite{jkt} directly extended the definitions and arguments from Bun et al.'s~\cite{bun} earlier $2^{O(2^d)}$-sample algorithm for the binary case. While plausible, it is currently unknown and far from obvious whether similar adaptations can be made to the improved binary algorithm of Ghazi et al.~\cite{ghazi}. Instead of attacking this problem directly, we show that Theorem~\ref{lem:generic-informal}, together with additional insights relating multiclass and binary Littlestone dimensions, allows us to \emph{generically} translate sample complexity upper bounds for private learning in terms of binary Littlestone dimension into upper bounds in terms of multiclass Littlestone dimension. Instantiating this general translation using the algorithm of Ghazi et al. gives the following improved sample complexity upper bound.


\begin{theorem}[Informal] \label{thm:sc-informal}
Let $H$ be a concept class consisting of $[k]$-valued functions and let $d$ be the multiclass Littlestone dimension of $H$. Then $H$ is privately learnable using $\tilde{O}(d^6 \log^8 (k+1))$ samples.
\end{theorem}

In addition to being conceptually simple and modular, our reduction from multiclass to binary learning means that potential future improvements for binary learning will also automatically give improvements for multiclass learning. For example, if one were able to prove that all binary classes of Littlestone dimension $d$ are privately learnable with 
$O(d)$ samples, this would imply that every $[k]$-valued class of multiclass Litttlestone dimension $d$ is privately learnable with $\tilde{O}(d \log^3 (k+1))$ samples.\footnote{The nearly cubic dependence on $\log (k+1)$ follows from the fact that the accuracy of private learners can be boosted with a sample complexity blowup that is nearly inverse linear in the target accuracy~\cite{drv10, bcs20}. See Theorem~\ref{thm:boosting}.}

\medskip

Finally, in Section~\ref{sec:puredp}, we study pure, i.e., $(\epsilon, 0)$-differentially private PAC learning in the multiclass setting. Beimel et al.~\cite{BNS19} characterized the sample complexity of pure private learners in the binary setting using the notion of \emph{probabilistic representation dimension}. We study a generalization of the representation dimension to the multiclass setting and show that it characterizes the sample complexity of pure private multiclass PAC learning up to a logarithmic term in the number of labels $k+1$. Our primary technical contribution in this section is a new and simplified proof of the relationship between representation dimension and Littlestone dimension that readily extends to the multiclass setting. This connection was previously explored by Feldman and Xiao \cite{feldman} in the binary setting, through a connection to randomized one-way communication complexity. We instead use techniques from online learning --- specifically, the experts framework and the weighted majority algorithm developed by Littlestone and Warmuth \cite{littlestonewar} for the binary setting and extended to the multiclass setting by Daniely et al. \cite{daniely1}. 

\subsection{Techniques}

Theorem~\ref{lem:generic-informal} shows that a multi-valued class $H$ is privately learnable if all of the binary classes $H|_i$ are privately learnable, which in turn holds as long as we can control their (binary) Littlestone dimensions. So the last remaining step in order to conclude Theorem~\ref{thm:sc-informal} is to show that if $H$ has bounded multiclass Littlestone dimension, then all of the classes $H|_i$ have bounded binary Littlestone dimension. At first glance, this may seem to follow immediately from the fact that (multiclass) Littlestone dimension characterizes (multiclass) online learnability --- a mistake bounded learner for a multiclass problem is, in particular, able to learn each individual output bit of the function being learned. The problem with this intuition is that the multiclass learner is given more feedback from each example, namely the entire multi-valued class label, than a binary learner for each $H|_i$ that is only given a single bit. Nevertheless, we are still able to use combinatorial methods to show that multiclass online learnability of a class $H$ implies online learnability of all of the binary classes $H|_i$.

\begin{theorem} \label{thm:online-reduction}
Let $H$ be a $[k]$-valued concept class with multiclass Littlestone dimension $d$. Then every binary class $H|_i$ has Littlestone dimension at most $6 d \ln (k+1)$.
\end{theorem}

Moreover, this result is nearly tight. In Section~\ref{sec:tightness}, we show that for every $k, d \ge 1$ there is a $[k]$-valued class with multiclass Littlestone dimension $d$ such that at least one of the classes $H|_i$ has Littlestone dimension at least $\Omega(d \log (k+1))$.

Theorem~\ref{thm:online-reduction} is the main technical contributions of this work. The proof adapts techniques introduced by Ben-David et al.~\cite{Ben} for characterizing the sample complexity of (non-private) multiclass PAC learnability. Specifically, Ben-David et al. introduced a family of combinatorial dimensions, parameterized by collections of maps $\Psi$ and called $\Psi$-dimensions, associated to classes of multi-valued functions. One choice of $\Psi$ corresponds to the ``Natarajan dimension''~\cite{nattu}, which was previously known to give a lower bound on the sample complexity of multiclass learnability. Another choice corresponds to the ``graph dimension''~\cite{nattu} which was known to give an upper bound. Ben-David et al. gave conditions under which $\Psi$-dimensions for different choices of $\Psi$ could be related to each other, concluding that the Natarajan and graph dimensions are always within an $O(\log (k+1))$ factor, and thus characterizing the sample complexity of multiclass learnability up to such a factor.

Our proof of Theorem~\ref{thm:online-reduction} proceeds by extending the definition of $\Psi$-dimension to online learning. We show that one choice of $\Psi$ corresponds to the multiclass Littlestone dimension, while a different choice corresponds to an upper bound on the maximum Littlestone dimension of any binary class $H|_i$. We relate the two quantities up to a logarithmic factor using a new variant of the Sauer-Shelah-Perles Lemma for the ``0-cover numbers'' of a class of multi-valued functions. While we were originally motivated by privacy, we believe that Theorem~\ref{thm:online-reduction} and the toolkit we develop for understanding online $\Psi$-dimensions may be of broader interest in the study of (multiclass) online learnability.

Next, in Section~\ref{sec:revdirection} we prove that the multiclass Littlestone dimension of any $[k]$-valued class $H$ can be no more than a $\log(k+1)$ multiplicative factor larger than the maximum Littlestone dimension over the classes $H|_i$. We also show that this is tight. Hence, our results give a complete characterization of the relationship between the multiclass Littlestone dimension of a class $H$ and the maximum Littlestone dimension over the corresponding binary classes $H|_i$.

Finally, we remark that Theorem~\ref{thm:online-reduction} implies a qualitative converse to Lemma~\ref{lem:generic-informal}. If a multi-valued class $H$ is privately learnable, then the lower bound of Jung et al.~\cite{jkt} implies that $H$ has finite multiclass Littlestone dimension. Theorem~\ref{thm:online-reduction} then shows that all of the classes $H|_i$ have finite binary Littlestone dimension, which implies via sample complexity upper bounds for binary private PAC learnability~\cite{bun, ghazi} that they are also privately learnable.


\section{Background}

\ifnum\neurips=0
\paragraph*{Differential Privacy}
\else
\paragraph*{Differential privacy.}
\fi
Differential privacy is a property of a randomized algorithm guaranteeing that the distributions obtained by running the algorithm on two datasets differing for one individual's data are indistinguishable up to a multiplicative factor $e^\epsilon$ and an additive factor $\delta$. Formally, it is defined as follows:
\begin{definition}[Differential privacy, \cite{DMNS06}]
Let $n\in\mathbb{N}$. A randomized algorithm $M:\mathcal{X}^n \to \mathcal{Y}$ is $(\epsilon, \delta)$-\emph{differentially private} if for all subsets $S \subseteq \mathcal{Y}$ of the output space, and for all datasets $X$ and $X'$ containing $n$ elements of the universe $\mathcal{X}$ and differing in at most one element (we call these neighbouring datasets), we have that
\[\Pr[M(X) \in S] \leq e^{\epsilon} \Pr[M(X') \in S] + \delta\]
\end{definition}
We will also need the closely related notion of $(\epsilon, \delta)$-indistinguishability of random variables.
\begin{definition}[$(\epsilon, \delta)$-indistinguishability]
Two random variables $a_1$ and $a_2$ defined over the same outcome space $\mathcal{Y}$ are said to be $(\epsilon, \delta)$-indistinguishable if for all subsets $S \subseteq \mathcal{Y}$,  we have that
\[\Pr[a_1 \in S] \leq e^{\epsilon} \Pr[a_2 \in S] + \delta\]
and
\[\Pr[a_2 \in S] \leq e^{\epsilon} \Pr[a_1 \in S] + \delta\]
\end{definition}
One useful property of differential privacy that we will use is that any output of a differentially private algorithm is closed under `post-processing', that is, its cannot be made less private by applying any data-independent transformations.
\begin{lemma}[Post-processing of differential privacy, \cite{DMNS06}]\label{prelim:postprocess} If $M: \mathcal{X}^n \to \mathcal{Y}$ is $(\epsilon, \delta)$-differentially private, and $\mathcal{B} : \mathcal{Y} \to \mathcal{Z}$ is any randomized function, then the algorithm $\mathcal{B} \circ M$ is $(\epsilon, \delta)$-differentially private.
\end{lemma}
Similarly, $(\epsilon, \delta)$-indistinguishability is also preserved under post-processing.
\begin{lemma}[Post-processing of $(\epsilon, \delta)$-indistinguishability]\label{prelim:postprocess2} If $a_1$ and $a_2$ are random variables over the same outcome space $\mathcal{Y}$ that are $(\epsilon, \delta)$-indistinguishable, then for any possibly randomized function $\mathcal{B}: \mathcal{Y} \to Z$, we have that $\mathcal{B}(a_1)$ and $\mathcal{B}(a_2)$ are $(\epsilon, \delta)$-indistinguishable.
\end{lemma}
\paragraph*{PAC learning.}
PAC learning~\cite{valiant} aims at capturing natural conditions under which an algorithm can approximately learn a hypothesis class.
\begin{definition}[Hypothesis class]
A \emph{hypothesis class} $H$ with input space $\mathcal{X}$ and output space $\mathcal{Y}$ (also called the label space) is a set of functions $f$ mapping $\mathcal{X}$ to $\mathcal{Y}$.
\end{definition}
Where it is clear, we will not explicitly name the input and output spaces. We can now formally define PAC learning.
\begin{definition}[PAC learning, \cite{valiant}]
A learning problem is defined by a hypothesis class $H$. For any distribution $P$ over the input space $\mathcal{X}$, consider $n$ independent draws $x_1, x_2, \dots x_n$ from distribution $P$. A labeled sample of size $n$ is the set $\{(x_1, f(x_1)), (x_2, f(x_2)), \dots, (x_n, f(x_n)) \}$ where $f \in H$. We say an algorithm $A$ taking a labeled sample $X$ of size $n$ is an $(\alpha, \beta)$-accurate PAC learner for the hypothesis class $H$ if for all functions $f \in H$ and for all distributions $P$ over the input space, $A$ on being given a labeled sample of size $n$ drawn from $P$ and labeled by $f$, outputs a hypothesis $h\in H$ such that with probability greater than or equal to $1 - \beta$ over the randomness of the sample and the algorithm,
\begin{equation*}
    \Pr[h(x) \neq f(x)] \leq \alpha.
\end{equation*}
\end{definition}
The definition above defines PAC learning in the \emph{realizable} setting, where all the functions $f$ labeling the data are in $H$. Two well studied settings for PAC learning are the \emph{binary learning} case, where $\mathcal{Y} = \{0, 1\}$ and the \emph{multiclass learning} case, where $\mathcal{Y} = [k]=\{0,1,\dots,k\}$ for natural numbers $k>2$. 
The natural notion of complexity for PAC learning is \emph{sample complexity}.

\begin{definition}[Sample complexity]
The sample complexity $S_{H, \alpha, \beta}(A)$ of algorithm $A$ with respect to hypothesis class $H$ is the minimum size of the sample that the algorithm requires in order to be an $(\alpha, \beta)$-accurate PAC learner for $H$. The PAC complexity of the hypothesis class $H$ is 
\[\inf_A S_{H, \alpha, \beta}(A).\]
\end{definition}
In this work, we will be interested in \textit{generic} learners, that work for every hypothesis class.
\begin{definition}[Generic learners]
We say that an algorithm $A$ that additionally takes the hypothesis class as an input, is a \textbf{generic} $(\alpha, \beta)$-accurate private PAC learner with sample complexity function $SC(H, \alpha, \beta)$, if for every hypothesis class $H$, it is an $(\alpha, \beta)$-accurate private PAC learner for $H$ with sample complexity $SC(H,\alpha, \beta)$.
\end{definition}

\ifnum\neurips=0
\subsection{Differentially Private PAC Learning}
\else
\paragraph*{Differentially private PAC learning.}
\fi
We can now define differentially private PAC learning, by putting together the constraints imposed by differential privacy and PAC learning respectively. 
\begin{definition}[Differentially private PAC learning \cite{KLNRS}]
An algorithm A is an $(\epsilon, \delta)$-differentially private and $(\alpha, \beta)$-accurate private PAC learner for the hypothesis class $H$ with sample complexity $n$ if and only if: 
\begin{enumerate}
    \item A is an $(\alpha, \beta)$-accurate PAC learner for the hypothesis class $H$ with sample complexity $n$.
    \item A is $(\epsilon, \delta)$-differentially private.
\end{enumerate}
\end{definition}
In this work, we study the complexity of private PAC learning. Our work focuses on the \textbf{multiclass realizable} setting. 

\ifnum\neurips=0
\subsection{Multiclass Littlestone Dimension}
\else
\paragraph*{Multiclass Littlestone dimension.}
\fi

We recall here the definition of multiclass Littlestone dimension~\cite{daniely1}, which we will use extensively in this work. Unless stated otherwise, we will use the convention that the root of a tree is at depth $0$.
As a first step, we define a class of labeled binary trees, representing possible input-output label sequences over an input space $\mathcal{X}$ and the label space $[k]$.
\begin{definition}[Complete io-labeled binary tree]
A \emph{complete io-labeled binary tree} of depth $b$ with input set $\mathcal{X}$ and output set $[k]$ consists of a complete binary tree of depth $b$ with the following properties:
\begin{enumerate}
    \item Every node of the tree other than the leaves is labeled by an example $x \in \mathcal{X}$. 
    \item The $2$ edges going from any parent node to its two children are labeled by two different labels in $[k]$. 
    \item The leaf nodes of the tree are unlabeled. 
\end{enumerate}
\end{definition}
We are interested in whether the input-ouput labelings defined by the complete io-labeled tree can be achieved by some function in the hypothesis class; to this end, we define realizability for root-to-leaf paths.
%
%
\begin{definition}
Given a complete io-labeled binary tree of depth $b$, consider a root-to-leaf path described as an ordered sequence $S=\{(x_i, y_i)\ | i\in [b]\}$, where $x_i$ is a node label and $y_i$ is the label of the edge between $x_i$ and $x_{i+1}$, and where $x_0$ is the root.  We say that the root-to-leaf path is \emph{realized} by a function $f \in H$ if for every 
$(x_i, y_i)$ in $S$, we have $x_i \in \mathcal{X}$ and $y_i=f(x_i)$.
\end{definition}
Using this definition we can now define what it means for a hypothesis class of functions to shatter a complete io-labeled binary tree, which helps to capture how expressive the hypothesis class is.
\begin{definition}[Shattering]
We say that a complete io-labeled binary tree of depth $b$ with label set $[k]$ is \emph{shattered} by a hypothesis class $H$ if for all $2^{b}$ root-to-leaf sequences $S$ of the tree, there exists a function $f \in H$ that realizes $S$. 
\end{definition}
Using this definition of shattering we can finally define the multiclass Littlestone dimension.
\begin{definition}[Multiclass Littlestone dimension, \cite{daniely1}]
The \textbf{multiclass Littlestone dimension} of a hypothesis class $H$, denoted $MLD(H)$, is defined to be the maximum $b$ such that there exists a  complete io-labeled binary tree of depth $b$ that is shattered by $H$. If no maximum exists, then we say that the multiclass Littlestone dimension of $H$ is $\infty$.
\end{definition}

\ifnum\neurips=0
\section{Main Results}
\else
\section{Main results}
\fi

\subsection{Reduction from multiclass private PAC learning to binary private PAC learning}
Our first main result is a reduction from multiclass private PAC learning to binary private PAC learning. Informally, the idea is that that every function $f$ mapping examples to labels in $[k]$ can be thought of as a vector of binary functions $(f_1, \dots, f_{\log (k+1)})$. Here, each binary function predicts a bit of the binary representation of the label predicted by $f$. Then, we can learn these binary functions by splitting the dataset into $\log(k+1)$ parts, and using each part to learn a different $f_i$. We can learn the binary functions using an $(\epsilon, \delta)$-DP binary PAC learner. Then, we can combine the binary hypotheses obtained to get a hypothesis for the multiclass setting, by applying a binary to decimal transformation. This process,  described in Figure~\ref{figure:reduction}, preserves privacy since changing a single element of the input dataset changes only one of the partitions, and we apply an $(\epsilon, \delta)$-DP learning algorithm to each partition. The binary to decimal transformation can be seen as post-processing. 

\begin{figure}[ht]
\centering
\includegraphics[width=0.9\linewidth]{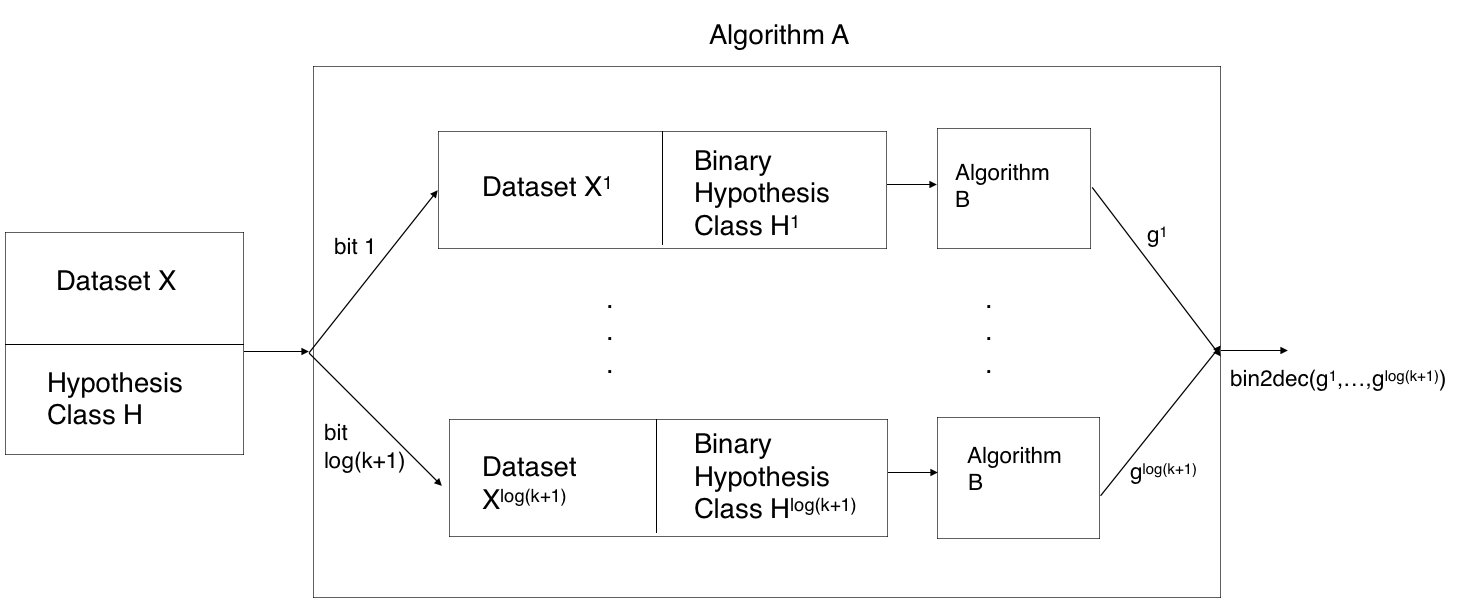}
\caption{Algorithm $A$ is the $(\epsilon, \delta)$-DP PAC learner for hypothesis classes with label set [k]. The algorithm $B$ used as a subroutine is an $(\epsilon, \delta)$-DP PAC learner for binary hypothesis classes. bin2dec represents a binary to decimal conversion.} \label{figure:reduction}
\centering
\end{figure}

Next, we formalize this idea. Given a hypothesis class $H$ with label set $[k]$, construct the following $\log (k+1)$ hypothesis classes $H|_{1}, \dots, H|_{\log (k+1)}$. For every function $f \in H$, let $f_i: \mathcal{X} \to \{0,1\}$ be the function defined such that $f_i(x)$ is the $i^{th}$ bit of the binary expansion of $f(x)$. Let the hypothesis class $H|_i$ be defined as $\{f_i: f \in H \}$. We will call these the \textbf{binary restrictions} of $H$.

\begin{theorem}\label{theorem:main}
Let $H$ be a hypothesis class with label set $[k]$ and let $H|_{1}, \dots, H|_{\log(k+1)}$ be its binary restrictions. Assume we have $(\epsilon, \delta)$-differentially private, $(\alpha, \beta)$-accurate PAC learners $B^1, \dots, B^{\log (k+1)}$ for $H|_{1}, \dots, H|_{\log(k+1)}$ with sample complexities upper bounded by $SC^1_{\alpha, \beta}, SC^2_{\alpha, \beta}, \dots,  SC^{\log(k+1)}_{\alpha, \beta}$. Then, there exists an $(\epsilon, \delta)$-differentially private, $(\alpha, \beta)$-accurate PAC learner $A$ for the hypothesis class $H$ that has sample complexity upper bounded by $\sum_{i=1}^{\log (k+1)} SC^i_{\alpha/\log (k+1), \beta/\log (k+1)}$.
\end{theorem}

\begin{proof}
For simplicity, let $k$ be a predecessor of a power of $2$. Note that if it is not, the argument below will work by replacing $k$ with the predecessor of the closest power of $2$ that is larger than $k$. 

Fix any distribution $P$ over $\mathcal{X}$ and an unknown function $f \in H = bin2dec(f_1, \dots, f_{\log(k+1)})$ ($bin2dec$ represents a binary to decimal conversion; which in this case will be an output in $[k]$) where $f_i \in H|_i$ predicts the $i^{th}$ bit of the binary expansion of the label predicted by $f$.

Assuming the algorithm is given a labeled sample $X$ of size $\sum_{i=1}^{\log(k+1)} SC^i_{\alpha/\log(k+1), \beta/\log(k+1)}$ drawn independently from $P$ and labeled by $f$, split the sample into $\log(k+1)$ smaller samples $X^1, \dots, X^{\log(k+1)}$. The first sample will be of size $SC^1_{\alpha/\log(k+1), \beta/\log (k+1)}$, the second sample will be of size $SC^2_{\alpha/\log(k+1), \beta/\log(k+1)}$ and so on. For each sample $X^i$, replace the labels of all examples in that sample by the $i^{th}$ bit of the binary expansion of what the label previously was. Note that this is equivalent to getting a sample of size $SC^i_{\alpha/\log(k+1), \beta/\log(k+1)}$ from distribution $P$ that is labeled by function $f_i \in H|_i$.

For all classes $H|_i$, $A$ runs the $(\epsilon, \delta)$-DP, $(\alpha / \log(k+1), \beta / \log(k+1))$-accurate PAC learning algorithm $B^i$ to learn $H|_{i}$ using the sample $X^i$. Let the hypothesis output when running the generic binary PAC learner on $H|_i$ be $g_i$. Then, $A$ outputs the function $g(x) = bin2dec(g_1(x), \dots, g_{\log(k+1)}(x))$.

First, we argue that $A$ is an $(\alpha, \beta)$-accurate PAC learner for $H$.
\begin{equation}\label{eq:private_pac}
    \Pr [g(x) \neq f(x) ] = \Pr[\exists i, g_i(x) \neq f_i(x)] \leq \sum_{i=1}^{\log(k+1)} \Pr[g_i(x) \neq f_i(x)]
\end{equation}
where the last inequality is by a union bound.

But since $g_i$ is the output of the $(\alpha/\log(k+1), \beta/\log(k+1))$-accurate PAC learner $B^i$ on $H|_i$, and we feed it a sufficient number of samples, we get that for any $i$, with probability $\geq 1 - \beta/\log(k+1)$,
\[\Pr[g_i(x) \neq f_i(x)] \leq \alpha / \log(k+1).\]
This means that again by a union bound, we can say that with probability $\geq 1 - \beta$,
\begin{equation}\label{eq:union_pac}
\forall i, \Pr[g_i(x) \neq f_i(x)] \leq \alpha / \log(k+1)
\end{equation}
Substituting equation~\ref{eq:union_pac} into equation~\ref{eq:private_pac}, we get that with probability $1-\beta$ over the randomness of the sample and the algorithm,
\begin{equation}
    \Pr [g(x) \neq f(x)] \leq \sum_{i=1}^{\log(k+1)} \Pr[g_i(x) \neq f_i(x)] \leq \alpha
\end{equation}
which means that $A$ is an $(\alpha, \beta)$-accurate PAC learner with sample complexity upper bounded by 
\[\sum_{i=1}^{\log(k+1)} SC^i_{\alpha/\log(k+1), \beta/\log(k+1)}.\]

We now argue that $A$ is $(\epsilon, \delta)$-DP. This will follow from the `parallel composition' property of $(\epsilon, \delta)$-DP.
\begin{claim}\label{claim:parcomp}
Let algorithm $M: \mathcal{X}^n \to \mathcal{Y}^r$ have the following structure: it splits its input data into $r$ disjoint partitions $X^1, X^2, \dots, X^r$ in a data-independent way. It runs $r$ (potentially different) $(\epsilon, \delta)$-DP algorithms $M^1: \mathcal{X}^* \to \mathcal{Y}, \dots, M^r: \mathcal{X}^* \to \mathcal{Y}$, one on each partition. It then outputs $\left( M^1(X^1), M^2(X^2), \dots, M^r(X^r) \right)$. Then, $M$ is $(\epsilon, \delta)$-DP.  
\end{claim}
\begin{proof}
Fix any two neighbouring datasets $X$ and $Y$. Then, we want to argue that the random variable $M(X)$ is $(\epsilon, \delta)$-indistinguishable from the random variable $M(Y)$. Observe that since $X$ and $Y$ differ in only one element, when we partition them, all but one partition is the same. Assume without loss of generality that only the first partition is different, that is $X^1 \neq Y^1$, but $X^2 = Y^2, \dots, X^r = Y^r$. $X^1$ and $Y^1$ are neighbouring datasets since they differ in only a single element. Hence since $M^1$ is $(\epsilon, \delta)$-DP, we have that $M^1(X^1)$ is $(\epsilon, \delta)$-indistinguishable from $M^1(Y^1)$.

Next, consider a randomized function $f_{X,Y}:\mathcal{Y} \to \mathcal{Y}^r$ (that depends on the neighbouring dataset pair) to represent the output of $M$ as follows: For any $q \in \mathcal{Y}$, let $f_{X,Y}(q) =\left( q, M^2(X^2), \dots, M^r(X^r)\right)$. By Claim~\ref{prelim:postprocess2}, since $(\epsilon, \delta)$-indistinguishability is preserved under post-processing, we have that $f_{X,Y}(M^1(X^1))$ is $(\epsilon, \delta)$-indistinguishable from $f_{X,Y}(M^1(Y^1))$. 

But \[f_{X,Y}(M^1(X^1)) = \left( M^1(X^1), M^2(X^2), \dots, M^r(X^r) \right) =M(X),\] and \[f_{X,Y}(M^1(Y^1)) =  \left( M^1(Y^1), M^2(X^2), \dots, M^r(X^r) \right) =  \left( M^1(Y^1), M^2(Y^2), \dots, M^r(Y^r) \right) = M(Y).\] where the second equality follows because $X^2 = Y^2, X^3 = Y^3, \dots, X^r = Y^r$. Hence, we get that $M(X)$ and $M(Y)$ are  $(\epsilon, \delta)$-indistinguishable. This argument works for any pair of databases; hence, we get that $M$ is $(\epsilon, \delta)$-DP.
\end{proof}
Note that algorithm $A$ follows a similar structure to that described in Claim~\ref{claim:parcomp}; it divides the dataset into $\log(k+1)$ partitions, runs an $(\epsilon, \delta)$-PAC learning algorithm for binary hypothesis classes on each partition and post-processes the outputs. Hence, by Claim~\ref{claim:parcomp} and by the fact that $(\epsilon, \delta)$-DP is closed under postprocessing (Claim~\ref{prelim:postprocess}), we get that $A$ is $(\epsilon, \delta)$-DP.
\end{proof}

Next, we recall that the sample complexity of privately learning binary hypothesis classes can be characterized by the Littlestone dimension of the hypothesis class \cite{almm, bun}. That is, there exists an $(\alpha, \beta)$-accurate, $(\epsilon, \delta)$-DP PAC learning algorithm for any binary hypothesis class $G$ with sample complexity upper and lower bounded by a function only depending on $\alpha, \beta, \epsilon, \delta$ and $d$ where $d$ is the Littlestone dimension of $G$. Using this characterization, we directly obtain the following corollary to Theorem~\ref{theorem:main}.
\begin{corollary}\label{cor:sampcompbinls}
Let $H$ be a hypothesis class with label set $[k]$ and let $H|_{1}, \dots, H|_{\log(k+1)}$ be its binary restrictions. Let the Littlestone dimensions of $H|_{1}, \dots, H|_{\log(k+1)}$ be $d_1, \dots, d_{\log(k+1)}$. Assume we have a generic $(\epsilon, \delta)$-differentially private, $(\alpha, \beta)$-accurate PAC learner $B$ for binary hypothesis classes $G$ that has sample complexity upper bounded by a function $SC_{\epsilon, \delta}(d', \alpha, \beta)$ where $d'$ is the Littlestone dimension of $G$.  Then, there exists an $(\epsilon, \delta)$-differentially private, $(\alpha, \beta)$-accurate PAC learner $A$ for $H$ that has sample complexity upper bounded by $\sum_{i=1}^{\log(k+1)} SC_{\epsilon, \delta}(d_i, \alpha/\log(k+1), \beta/\log(k+1))$.
\end{corollary}
Corollary~\ref{cor:sampcompbinls} shows that the sample complexity of privately PAC learning a hypothesis class in the multiclass setting can be upper bounded by a function depending on the Littlestone dimensions of its binary restrictions. However, as described earlier, Jung et al. \cite{jkt} showed that the sample complexity of private multiclass PAC learning could be characterized by the multiclass Littlestone dimension. Hence, an immediate question is what the relationship between the multiclass Littlestone dimension of a class and the Littlestone dimensions of its binary restrictions is.

\ifnum\neurips=0
\subsection{Connection between Multiclass and Binary Littlestone Dimension}
\else
\subsection{Connection between multiclass and binary Littlestone dimension}
\fi

We show that the multiclass Littlestone dimension $MLD(H)$ of a hypothesis class is intimately connected to the maximum Littlestone dimension over its binary restrictions.

\begin{theorem}\label{theorem:lsd}
Let $H$ by a hypothesis class with input set $\mathcal{X}$ and output set $[k]$. Let the multiclass Littlestone dimension of $H$ be $d$. Let $H|_1, H|_2, \dots, H|_{\log(k+1)}$ be the binary restrictions of $H$. Let the Littlestone dimensions of $H|_1, H|_2, \dots, H|_{\log(k+1)}$ be $d_1, \dots, d_{\log(k+1)}$. Then, 
\[\max_{i=1, \dots, \log(k+1)} d_i \leq 6 d \ln(k+1).\]
\end{theorem}
A similar-looking theorem relating the Natarajan dimension of a hypothesis class with the maximum VC dimension over its binary restrictions was proved in Ben-David et al. \cite{Ben} using the notion of $\Psi$-dimension. Our proof of Theorem~\ref{theorem:lsd} is inspired by this strategy. It will proceed by defining and analyzing a notion of dimension that we call $\Psi$-Littlestone dimension. It will also use the $0$-cover function of a hypothesis class defined in Rakhlin et al. \cite{RST15}. The details of the proof are described in Section~\ref{sec:lsdproof}. 

This theorem is tight; for all $d \geq 0$ and $k \geq 1$, there exists a hypothesis class $H$ with label set $[k]$ and multiclass Littlestone dimension $d$ such the maximum Littlestone dimensions over the binary restrictions of $H$ is $O(d \ln(k+1))$. We prove this in Section~\ref{sec:tightness}. Additionally, the reverse direction is also true, the multiclass Littlestone dimension of any hypothesis class $H$ with label set $[k]$ is at most a $\log(k+1)$ factor larger than the maximum Littlestone dimension over its binary restrictions (this is also tight). We prove this in Section~\ref{sec:revdirection}. 

These arguments together completely describe the relationship between the multiclass Littlestone dimension of a hypothesis class $H$ with label set $[k]$ and the maximum Littlestone dimension over its binary restrictions. 

Finally, combining Theorem~\ref{theorem:lsd} and Corollary~\ref{cor:sampcompbinls}, we can directly obtain the following corollary to Theorem~\ref{theorem:main}. 

\begin{corollary}\label{cor:sampcompMLD}
Assume we have a generic $(\epsilon, \delta)$-differentially private, $(\alpha, \beta)$-accurate PAC learner $B$ for binary hypothesis classes $G$ that has sample complexity upper bounded by a function $SC_{\epsilon, \delta}(d', \alpha, \beta)$ where $d'$ is the Littlestone dimension of $G$.  Then, there exists a generic $(\epsilon, \delta)$-differentially private, $(\alpha, \beta)$-accurate PAC learner $A$ for multi-valued hypothesis classes $H$ (label set $[k]$) that has sample complexity upper bounded by $\log(k+1) SC_{\epsilon, \delta}(6d \ln(k+1), \alpha/\log (k+1), \beta/\log(k+1))$ where $d$ is the multiclass Littlestone dimension of $H$.
\end{corollary}

We now consider an application of this result. The best known sample complexity bound for $(\epsilon, \delta)$-DP binary PAC learning is achieved by a learner described in Ghazi et al. \cite{ghazi}. We state a slightly looser version of their result here.
\begin{theorem}[Theorem 6.4 \cite{ghazi}]\label{thm:ghazi}
Let $G$ be any binary hypothesis class with Littlestone dimension $d_L$. Then, for any $\epsilon, \delta, \alpha, \beta \in [0,1]$,
for some \[n = O\left(\frac{d_L^6 \log^2(\frac{d_L}{\alpha \beta \epsilon \delta})}{\epsilon \alpha^2}\right),\]
there is an $(\epsilon, \delta)$-differentially private, $(\alpha, \beta)$-accurate PAC learning algorithm $B$ for $G$ with sample complexity upper bounded by $n$.
\end{theorem}
Now, applying the reduction described in Theorem~\ref{theorem:main}, with this learner as a subroutine, we get the following theorem.  (Instead of directly applying Theorem~\ref{thm:ghazi}, we will instead first use a boosting procedure described in Appendix~\ref{sec:Boosting}.)
\begin{theorem}\label{thm:mainsampcomp}
Let $H$ be a concept class over $\mathcal{X}$ with label set $[k]$ and multiclass Littlestone dimension $d$. Then, for any $\epsilon \in [0,1/4], \delta, \alpha, \beta \in [0,1]$, for some \[n = O\Big(\frac{d^6 (\log(k+1))^8 \log^4(\frac{d \log^3(k+1)}{\epsilon \delta \alpha \beta})}{\epsilon \alpha}\Big)\]
there is an $(\epsilon, \delta)$-differentially private, $(\alpha, \beta)$-accurate PAC learning algorithm $A$ for $H$ with sample complexity upper bounded by $n$.
\end{theorem}
\begin{proof}
We will use the fact that the binary PAC learner from Ghazi et al. can be boosted to give a learner for binary hypothesis classes $H$ with Littlestone dimension $d'$ with sample complexity  upper bounded by $O\left(\frac{d'^6 \log^4(\frac{d'}{\alpha \beta \epsilon \delta})}{\epsilon \alpha}\right)$. The main difference is that the sample complexity is nearly inverse linear in the term $\alpha$ versus inverse quadratic. This boosting procedure is discussed in detail in Section~\ref{sec:Boosting} and the sample complexity bound we use here is derived in Corollary~\ref{thm:ghazi-2}. 

Substituting into Corollary~\ref{cor:sampcompMLD} with $d' = 6d\ln(k+1), \alpha' = \alpha/\log(k+1), \beta' = \beta/\log(k+1)$ gives the result.   
\end{proof}
\ifnum\neurips=0
\section{$\Psi$-Littlestone Dimension}
\else
\section{$\Psi$-Littlestone dimension}
\fi

\subsection{Definition}\label{sec:psidef}
In this section, we define an online analog of the $\Psi$-dimension \cite{Ben} that will help us prove Theorem~\ref{theorem:lsd}. The main intuition is that similar to in the definition of $\Psi$-dimension, we can use what we term \textit{collapsing maps} to reason about the multiclass setting while working with binary outputs. Let $\phi: [k] \to \{0, 1, *\}$ represent a function that maps labels to $\{0,1, * \}$, which we call a \textit{collapsing map}. We refer to a set of collapsing maps $\Psi$ as a \emph{family}. The definitions of labeled trees will be the only distinction from the regular definition of multiclass Littlestone dimension, and every node will have not only an example, but also a collapsing map assigned to it. 
\begin{definition}[$\Psi$-labeled binary tree]
A complete $\Psi$-labeled binary tree of depth $b$ with label set $[k]$ and mapping set $\Psi$ on input space $\mathcal{X}$ consists of a complete binary tree of depth $b$ with the following labels:
\begin{enumerate}
    \item Every node of the tree other than the leaves is labeled by an example $x \in \mathcal{X}$, and a collapsing map $\phi \in \Psi$. 
    \item The left and right edges going from any parent node to its two children are labeled by $0$ and $1$ respectively. 
    \item The leaf nodes of the tree are unlabeled. 
\end{enumerate}
A complete $\Psi$-uniformly labeled binary tree of depth $b$ with label set $[k]$ and mapping set $\Psi$ on input space $\mathcal{X}$ is defined in the same way, with the additional property that all nodes at the same depth are labeled by the same collapsing map.
\end{definition}
Where the input space, label space and mapping set are obvious, we will omit them and simply refer to a complete $\Psi$-labeled binary tree or  $\Psi$-uniformly labeled binary tree. 
\begin{definition}
Consider a root-to-leaf path in a complete $\Psi$-labeled binary tree described as an ordered sequence $S = ((x_0, \phi_0, y_0), \dots, (x_{b-1}, \phi_{b-1}, y_{b-1}))$, where each $x_i \in \mathcal{X}$ is an input, $\phi_i$ is a collapsing map, and $y_i \in \{0, 1\}$ is an edge label. We say that this path is realized by a function $f \in H$ if $y_i = \phi_i(f(x_i))$ for every triple in the ordered sequence $S$.
\end{definition}
We can now define what it means for a class of functions to $\Psi$-shatter a complete $\Psi$-labeled binary tree. 
\begin{definition}[$\Psi$-shattering]
We say that a complete $\Psi$-labeled binary tree of depth $b$ with label set $[k]$ is $\Psi$-shattered by a hypothesis class $H$ if for all $2^{b}$ root-to-leaf sequences $S$ of the tree, there exists a function $f \in H$ that realizes $S$. 
Similarly, we say that a complete binary $\Psi$-uniformly labeled tree of depth $b$ with label set $[k]$ is $\Psi$-shattered by a hypothesis class $H$ if for all $2^{b}$ root-to-leaf sequences $S$ of the tree, there exists a function $f \in H$ that realizes $S$.
\end{definition}

Finally, we are in a position to define the $\Psi$-Littlestone dimension. 

\begin{definition}[$\Psi$-Littlestone dimension]
The \textbf{$\Psi$-Littlestone dimension} $\Psi_{LD}(H)$ of a hypothesis class $H$ is defined to be the maximum depth $b$ such that there is a complete $\Psi$-labeled binary tree of depth $b$ that is $\Psi$-shattered by $H$. If no maximum exists, then we say that the $\Psi$-Littlestone dimension of $H$ is $d=\infty$. The \textbf{uniform $\Psi$-Littlestone dimension} $\Psi_{LD} U(H)$ is defined similarly (using the definition of $\Psi$-shattering for complete $\Psi$-uniformly labeled binary trees instead).
\end{definition}

\ifnum\neurips=0
\subsection{Properties of $\Psi$-Littlestone Dimension}\label{sec:prop}
\else
\subsection{Properties of $\Psi$-Littlestone dimension}\label{sec:prop}
\fi

In this section, we begin our investigation of the $\Psi$-Littlestone dimensions by discussing a few simple and useful properties. We first define three important families of collapsing maps $\Psi^N$, $\Psi^{bin}$ and $\Psi^{B}$ that will play an important role in our results.

Consider a collapsing map $\phi_{w,w'}$ defined by $\phi_{w,w'}(\ell) = 0$ if $\ell=w$, $\phi_{w,w'}(\ell) = 1$ if $\ell=w'$, and $\phi_{w,w'}(\ell) = *$ otherwise. Then, $\Psi^N$ is defined to be $\{ \phi_{w,w'} | w \neq w', w,w' \in [k] \}$. Similarly, let $\phi_i$ be a collapsing map that maps a label in $[k]$ to the $i^{th}$ bit of its $\log(k+1)$-bit binary expansion. Then, $\Psi^{bin}$ is defined to be $\{ \phi_{i} \text{ }|\text{ }i = 1, \dots, \log(k+1) \}$. Finally, $\Psi^B$ is defined as the family of all collapsing maps from $[k]$ to $\{0,1,*\}$.

We first show that the multiclass Littlestone dimension of a hypothesis class $H$ (denoted $MLD(H)$) is equivalent to $\Psi^N_{LD}(H)$. 
\begin{lemma}\label{lem:olnatdim}
For all hypothesis classes $H$,  $\Psi^N_{LD}(H) = MLD(H)$.
\end{lemma}
\begin{proof}
Consider any complete io-labeled binary tree $T$ of depth $MLD(H)$ that is shattered by $H$. Construct a complete $\Psi^N$-labeled binary tree $T'$ as follows. The tree will be of the same depth as $T$. If in $T$, for a particular parent node, the two edges from a parent to a child are labeled by $w,w'$, then let the collapsing map labeling the parent node in $T'$ be $\phi_{w,w'}$. The edge labeled $w$ in $T$ will be labeled by $0$ in $T'$ and the other edge will be labeled by $1$. Also, label the nodes of $T'$ with examples in exactly the same way as $T$. The leaves remain unlabeled. By the definition of shattering, for every root-to-leaf path in $T$, there is a function that realizes that path. This function will continue to realize the corresponding path in $T'$. Hence, $T'$ is $\Psi^N$-shattered by $H$. This implies that
\[ MLD(H) \leq \Psi^N_{LD}(H).\]
The other direction performs this construction in reverse: it takes a complete $\Psi^N$-labeled binary tree $T'$ that is $\Psi^N$-shattered by $H$ and creates a complete io-labeled binary tree $T$ of the same depth that is shattered by $H$. For any node in $T'$, if the collapsing map assigned to that node is $\phi_{k,k'}$, the edges of that node to its children in $T$ will be labeled $k$ and $k'$ respectively (the edge labeled $0$ in $T'$ will be labeled by $k$ in $T$ and the other edge will be labeled by $k'$). The nodes of $T$ are labeled with the same examples as $T'$. The leaves remain unlabeled. By a similar argument to that in the previous paragraph, we have that $T$ is shattered by $H$, which means that
\[ \Psi^N_{LD}(H) \leq MLD(H).\]
This proves the claim.
\end{proof}

Next, we connect the Littlestone dimension of the binary restrictions of a hypothesis class $H$ with label set $[k]$ to the $\Psi^{bin}$-Littlestone dimension of the class.
\begin{claim}\label{lem:lsbin}
Consider any hypothesis class $H$ with label set $[k]$, and let $H|_1, H|_2, \dots, H|_{\log(k+1)}$ be the binary restrictions of $H$. Let the Littlestone dimension of $H|_j$ be $d_j$. Then,
\[\max_j d_j \leq \Psi^{bin}_{LD} U(H) \leq \Psi^{bin}_{LD}(H).\]
\end{claim}
\begin{proof}
The second inequality follows immediately from the fact that for any $\Psi$, if there exists a complete $\Psi$-uniformly labeled binary tree that is $\Psi$-shattered by $H$, then there exists a complete $\Psi$-labeled binary tree that is $\Psi$-shattered by $H$.

To prove the first inequality, fix a class $H|_i$ such that $d_i = \max_j d_j$. Consider a complete, io-labeled binary tree $T$ of depth $d_i$ that is shattered by $H|_i$. Then, construct the following complete $\Psi^{bin}$-labeled binary tree $T'$ of the same depth $d_i$. For every node, label it with the same example as in tree $T$. Every node in $T'$ is labeled with the collapsing map $\phi_i$ which maps a label to the $i^{th}$ bit of its binary expansion. The leaves remain unlabeled. Then, we have that $H$ $\Psi^{bin}$-shatters $T'$. Additionally, $T'$ is of the same depth as $T$ and all nodes at the same depth are labeled by the same collapsing map. Hence, \[\max_j d_j \leq \Psi^{bin}_{LD} U(H).\] \end{proof}

Finally, we relate the notions of $\Psi$-Littlestone dimension we have obtained with the families $\Psi^N$, $\Psi^B$ and $\Psi^{bin}$.
\begin{claim}\label{lem:psilsdrel}
For all hypothesis class $H$,
\[\Psi^N_{LD}(H) \leq \Psi^{bin}_{LD}(H) \leq \Psi^{B}_{LD}(H).\]
\end{claim}
\begin{proof}
Consider any complete $\Psi^N$-labeled binary tree of depth $\Psi^N_{LD}(H)$ that is $\Psi^N$-shattered by $H$. Construct a complete $\Psi^{bin}$-labeled binary tree $T'$ of the same depth as follows.  Label the nodes of $T'$ with examples exactly as in $T$. Consider a node in $T$ and the collapsing map $\phi_{w,w'}$ that labels the node. There is at least one bit in which the binary expansions of $w$ and $w'$ vary. Let this bit be the $i^{th}$ bit. Then, label the corresponding node in $T'$ with the collapsing map $\phi_i$, which maps every label to the $i^{th}$ bit of its binary expansion. Consider the two edges emanating from this node. If the $i^{th}$ bit of the binary expansion of $w$ is $0$, then in $T'$, label the edge that was labeled $0$ in $T$ by $0$ and the other by $1$. Else, label the edge that was labeled $0$ in $T$ by $1$ and the other by $0$. Perform this transformation for every labeled node in $T$ to obtain a corresponding labeled node in $T'$. The leaves of $T'$ will remain unlabeled.

Then, $T'$ is $\Psi^{bin}$-shattered by $H$. This gives that $\Psi^N_{LD}(H) \leq \Psi^{bin}_{LD}(H)$. The second inequality follows because $\Psi^{bin} \subseteq \Psi^B$, and so a $\Psi^{bin}$-labeled tree that is $\Psi^{bin}$-shattered by $H$ is automatically also a $\Psi^B$-labeled tree that is  $\Psi^{B}$-shattered by $H$.
\end{proof}


\ifnum\neurips=0
\section{Proof of Theorem~\ref{theorem:lsd}}\label{sec:lsdproof}
\else
\section{Proof of theorem~\ref{theorem:lsd}}\label{sec:lsdproof}
\fi
In this section, we use the concept of $\Psi$-Littlestone dimension to prove Theorem~\ref{theorem:lsd}. 

\ifnum\neurips=0
\subsection{Sauer's Lemma for Multiclass Littlestone Dimension}
\else
\subsection{Sauer's lemma for multiclass Littlestone dimension}
\fi

In this section, we will describe a version of Sauer's Lemma that will suffice for our application. This argument is essentially due to Rakhlin et al. \cite{RST15}. Theorem 7 in that paper states a Sauer's lemma style upper bound for a quantity they introduce called the ``0-cover function'', for hypothesis classes with bounded ``sequential fat-shattering dimension.''  We show that this argument applies almost verbatim for hypothesis classes with bounded multiclass Littlestone dimension. 

\ifnum\neurips=0
\subsubsection{$\mathbf{0}$-Cover Function}
\else
\subsubsection{$\mathbf{0}$-Cover function}
\fi

We start by recalling the definition of 0-cover from Rakhlin et al.
\begin{definition}[output-labeled trees, input-labeled trees]
A complete output-labeled binary tree of depth $b$ with label set $[k]$ is a complete binary tree of depth $b$ such that every node of the tree is labeled with an output in $[k]$. A complete input-labeled binary tree of depth $b$ with input set $\mathcal{X}$ is a complete binary tree of depth $b$ such that every node of the tree is labeled with an input in $\mathcal{X}$.
\end{definition}
The convention we will use is that output and input-labeled binary trees have root at depth $1$ (as opposed to io-labeled trees and $\Psi$-labeled trees, where we use the convention that root has depth $0$). 
Consider a set $V$ of complete output-labeled binary trees of depth $b$ with label set $[k]$. Consider a hypothesis class $H$ consisting of functions from input space $\mathcal{X}$ to label set $[k]$. Fix a complete input-labeled binary tree $z$ of depth $b$ with input space $\mathcal{X}$ and a complete output-labeled tree $v \in V$. 
\begin{definition}
We say that a root-to-leaf path $A$ in $z$ \textbf{corresponds} to a root-to-leaf path $B$ in $v$ if for all $1 \leq i \leq b-1 $, if node $i+1$ in $A$ is the left child of node $i$ in $A$, then node $i+1$ in $B$ is the left child of node $i$ in $B$ and likewise for the case where node $i+1$ is the right child of node $i$. 
\end{definition}
\begin{definition}
Let $A$ be a root-to-leaf path in $z$ and let the the labels of the nodes in $A$ be $(x_1, \dots, x_b)$ where $x_i \in \mathcal{X}$.
The function $f \in H$ applied to $A$, denoted by $f(A)$, is the sequence $\left( f(x_1), \dots, f(x_b) \right)$.
\end{definition}
\begin{definition}[$0$-cover, \cite{RST15}]
We say that $V$ forms a \textbf{0-cover} of hypothesis class $H$ on tree $z$ if, for every function $f \in H$ and every root-to-leaf path $A$ in $z$, there exists a complete output-labeled tree $v \in V$, such that for the corresponding root-to-leaf path $B \in v$ with the labels of nodes in $B$ denoted by a tuple $C \in [k]^b$ (call this the \textbf{label sequence} of $B$), we have that $f(A) = C$.   
\end{definition}
\begin{definition}[$0$-cover function, \cite{RST15}]
Let $N(0,H,z)$ denote the size of the smallest $0$-cover of hypothesis class $H$ on tree $z$. Let $T_b^{\mathcal{X}}$ be the set of all complete input-labeled binary trees of depth $b$ with input space $\mathcal{X}$. Then, the \textbf{0-cover function} $N(0,H,b)$ of hypothesis class $H$ is defined as $\sup_{z \in T_b^{\mathcal{X}}} N(0,H,z)$.
\end{definition}
We use the convention that $N(0, H, 0) = 1$.
\subsubsection{Statement of theorem}
The following theorem is essentially Theorem 7 of Rakhlin et al.~\cite{RST15} (with multiclass Littlestone dimension in place of sequential fat shattering dimension).
\begin{theorem}\label{thm:sauer}
Let hypothesis class $H$ be a set of functions $f: \mathcal{X} \to [k]$. Let the multiclass Littlestone dimension of $H$ be $d$. Then, for all natural numbers $n \geq d$, with $d \geq 0$,
\begin{equation}\label{eq:sauer}
N(0,H,n) \leq \sum_{i=0}^d {n \choose i} k^i
\end{equation}
For all natural numbers $n \geq d$, with $d > 0$, we additionally have the following:
\begin{equation}\label{eq:sauer2}
N(0,H,n) \leq \sum_{i=0}^d {n \choose i} k^i \leq \left(\frac{ekn}{d}\right)^d.
\end{equation}
Finally, for all $d \geq 0$, for all natural numbers $n < d$, we have $N(0,H,n) \leq (k+1)^n$. 
\end{theorem}
\begin{proof}
Firstly, observe that for all $n \geq d$, $d>0$,
\begin{align*}
    \sum_{i=0}^d {n \choose i} k^i 
    & = \sum_{i=0}^d {n \choose i} k^i \left(\frac{n}{d}\right)^i \left(\frac{d}{n}\right)^i \\
    & \leq \left(\frac{kn}{d}\right)^d \sum_{i=0}^d {n \choose i} \left(\frac{d}{n}\right)^i \\
    & \leq \left(\frac{kn}{d}\right)^d \left(1+\frac{d}{n}\right)^n \\
    & \leq \left(\frac{ekn}{d}\right)^d.
\end{align*}
This proves the second inequality in expression~\ref{eq:sauer2}.

The proof of the rest of the theorem will be by double induction on $n$ and $d$.

\paragraph{First base case $\bm{(d=0, n \geq 1)}$:} Observe that when $d=0$, the class $H$ consists of only a single distinct function. Call this function $f$. Then, for any complete, input-labeled binary tree $z$ of depth $n$ on input set $\mathcal{X}$, create a complete, output-labeled binary tree $v$ of depth $n$ on output set $[k]$ as follows: for every node in $z$ labeled by input $x \in \mathcal{X}$, label the corresponding node in $v$ by $f(x)$. Then the set consisting of just one tree $v$ is a $0$-cover for $z$. Thus we have that $N(0,H,n) = 1 = \sum_{i=0}^0  {n \choose i} k^i $, verifying this base case.

\paragraph{Second base case $\bm{(0 < n \leq d)}$:} 
We will prove a stronger statement; we will show that for any complete input-labeled binary tree $z$ of depth $n$ (for any natural number $n$), there is a $0$-cover of hypothesis class $H$ on $z$ of size $(k+1)^n$. This also proves the final part of the theorem corresponding to $n < d$. We start by observing that there are $(k+1)^n$ sequences of $n$ elements from $[k]$. For every such sequence, create a complete output-labeled binary tree $v$ of depth $n$ as follows: label all nodes at depth $i$ by the $i^{th}$ element of the sequence. In this way, we create $(k+1)^n$ different trees. This set of trees $V$ will form a $0$-cover for $H$ on $z$. To see this, fix a root-to-leaf path $A$ in $z$ and a function $f \in H$ and consider the sequence  $f(A) \in [k]^n$. Then by construction, there is a tree $v \in V$ such that every root-to-leaf path $B$ in $v$ has label sequence $f(A)$. This implies that $V$ is a $0$-cover of hypothesis class $H$ on $z$. Thus, we have that $N(0,H,n) \leq (k+1)^n = \sum_{i=0}^n {n \choose i} k^i \leq \sum_{i=0}^d {n \choose i} k^i$ for $n \leq d$, verifying the second base case.

\paragraph{Inductive case:} Fix a $d \geq 1,n \geq 2$ such that $n>d$ (note that the base cases handle other values of $d$ and $n$). Assume that the theorem is true for all pairs of values $(d', n')$ where $d' \leq d$ and $n' \leq n-1$. We will prove it is true for values $d,n$. Consider a complete, input-labeled binary tree $z$ of depth $n$ with input set $\mathcal{X}$. Let the root node of $z$ be labeled by example $x_r \in \mathcal{X}$. Divide hypothesis class $H$ into $k+1$ subclasses $H^0, \dots, H^k$ as follows, 
\[H^i = \{f \in H: f(x_r) = i\}.\]
That is, $H^i$ is the subclass of functions in $H$ that output label $i$ on example $x_r$. 
\begin{claim}\label{claim:MLDsub}
There exists at most one $i \in [k]$ such that $H^i$ has multiclass Littlestone dimension $d$. Every other subclass has multiclass Littlestone dimension at most $d-1$. 
\end{claim}
\begin{proof}
 Assume by way of contradiction that there are two hypothesis classes $H^i$ and $H^j$ that both have multiclass Littlestone dimension $d$. Then there are complete io-labeled, binary trees $T^L$ and $T^R$ of depth $d$ with input set $\mathcal{X}$ and output set $[k]$ that are shattered by $H^i$ and $H^j$ respectively. Construct a complete io-labeled binary tree $T$ of depth $d+1$ with input set $\mathcal{X}$ and output set $[k]$ as follows: set the root node to be $x_r$, and label the two edges emanating from the root by $i$ and $j$ respectively. Set the left sub-tree of the root to be $T^L$ and the right sub-tree to be $T^R$. Then $H$ shatters $T$ since $H^i$ and $H^j$ shatter $T^L$ and $T^R$ respectively. However, this is a contradiction since $H$ has multiclass Littlestone dimension $d$, and the shattered tree $T$ has depth $d+1$.  
\end{proof}
Next, consider any hypothesis class $H^i$ with multiclass Littlestone dimension equal to $d$. If no such class exists, simply choose the class $H^i$ with maximum multiclass Littlestone dimension (note that $MLD(H^i)$ will be upper bounded by $d$). Let $z^L$ and $z^R$ be the left and right sub-trees of depth $n-1$ of the root of $z$. By the inductive hypothesis, there are $0$-covers $V^L$ and $V^R$ of $H^i$ on $z^L$ and $z^R$ each of size at most $\sum_{i=0}^{d} {n-1 \choose i} k^i$.  We will now stitch together trees from $V^L$ and $V^R$ to create a set of trees $V$ that will form a $0$-cover of $H^i$ on $z$. Informally, we do this as follows. Every tree in $V$ will have root labeled by $i$. The left sub-tree of the root will be assigned to be some tree from $V^L$ and the right sub-tree of the root will be assigned to be some tree from $V^R$.  

Formally, without loss of generality, let $|V^L| \geq |V^R|$. Then, there exists a surjective function $\xi$ from $V^L$ to $V^R$. For every tree $v^L \in V^L$, construct a tree in $V$ as follows, the root will be labeled by $i$, the left subtree will be $v^L$ and the right subtree will be labeled by $\xi(v^L)$. Clearly, the size of $V$ is equal to the size of $V^L$, which is at most $\sum_{i=0}^{d} {n-1 \choose i} k^i$. 
Next, we argue that the set $V$ is a $0$-cover for $H^i$ on $z$. 
\begin{claim}
$V$ is a $0$-cover of $H^i$ on $z$.
\end{claim}
\begin{proof}
Fix a root-to-leaf path $A$ in $z$ and fix a function $f \in H^i$.  Let $A = (x_r, A|_{2:n})$ where $A|_{2:n}$ is the root-to-leaf path omitting the root. Consider $f(A) = (f(x_r), f(A|_{2:n})) = (i, f(A|_{2:n}))$. Note that $A|_{2:n}$ is a root-to-leaf path of either $z^L$ or $z^R$. Without loss of generality, assume it is a root-to-leaf path of $z^L$. Hence, there exists a tree $v^L$ in $V^L$ such that the root-to-leaf path $B$ in $v^L$ corresponding to root-to-leaf path $A|_{2:n}$ in $z^L$ has label sequence $C$ such that $f(A|_{2:n}) = C$. (This is true since $V^L$ is a $0$-cover of $H^i$ on $z^L$). By the construction of $V$, there is a tree $v$ in $V$ that has a root-to-leaf path $B'$ (with label sequence $C'$) corresponding to root-to-leaf path $A$ in $z$ such that $f(A) = (i, f(A|_{2:n})) = (i,C) = C'$. Hence, we have that $V$ is a $0$-cover of $H^i$ on $z$. 
\end{proof}
A very similar argument can also be used to construct $0$ covers of size at most $\sum_{i=0}^{d-1} {n-1 \choose i} k^i$ for  hypothesis classes $H^i$ with multiclass Littlestone dimension at most $d-1$. 

We now use the fact that the covers we constructed for $H^0, \dots, H^k$ can be combined into a cover for $H$.
\begin{claim}\label{claim:unioncover}
Let hypothesis class $G = G^1 \cup G^2 \cup \cdots \cup G^l$ for some positive integer $l$. Let $V^1, \dots, V^l$ be $0$-covers of $G^1, \dots, G^l$ on $z$. Then $V' = V^1 \cup \cdots \cup V^l$ is a $0$-cover of $G$ on $z$.
\end{claim}
\begin{proof}
Consider $f \in G$. Then, there exists an $i \in \{1,\dots,l\}$ such that $f \in G^i$. Thus, for every root-to-leaf path in $z$, there is a tree $v \in V^i$ that is consistent with $f(z)$. By the definition of $V'$, $v \in V'$ as well. This argument works for any function $f \in G$. Hence $V'$ is a $0$-cover of $G$ on $z$.
\end{proof}

Using Claim~\ref{claim:unioncover}, we can construct a $0$-cover $V'$ for $H$ on $z$ by taking the union of the $0$-covers for $H^0, \dots, H^k$ on $z$. This means that the size of $V'$ is less than or equal to the sums of sizes of the $0$-covers for $H^0, \dots, H^k$ on $z$. Additionally, by Claim~\ref{claim:MLDsub}, we have that at least $k$ of the hypothesis classes $H^i$ have multiclass Littlestone dimension at most $d-1$. This implies that
\[|V'| \leq  k \sum_{i=0}^{d-1} {n-1 \choose i} k^i + \sum_{i=0}^{d} {n-1 \choose i} k^i.\]
Finally, we simplify this using the following claim:
\begin{claim}
For all natural numbers $k,n$ and for all integers $d$ such that $0 < d < n$,
\[
 \sum_{i=0}^{d} {n-1 \choose i} k^i+ k \sum_{i=0}^{d-1} {n-1 \choose i} k^i = \sum_{i=0}^{d} {n \choose i} k^i.
\]
\end{claim}
\begin{proof}
Regrouping the terms in the sums, and using the fact that ${n \choose i } = {n-1 \choose i-1 } + {n-1 \choose i }$, we get that
\begin{align*}
    \sum_{i=0}^{d}& {n-1 \choose i} k^i+ k \sum_{i=0}^{d-1} {n-1 \choose i} k^i \\
    &=   \left({n-1 \choose d} k^d +  {n-1 \choose d-1}k \cdot k^{d-1} \right) + \left({n-1 \choose d-1} k^{d-1} +  {n-1 \choose d-2}k \cdot k^{d-2} \right) + \dots \\
    & \qquad \dots + \left({n-1 \choose 1} k^{1} + {n-1 \choose 0} k \cdot k^{0} \right) + {n \choose 0}k^0  \\
  &  = {n \choose d}k^d + {n \choose d-1}k^{d-1} + \cdots + {n \choose 0}k^0 \\
  &  = \sum_{i=0}^{d} {n \choose i} k^i.
\end{align*}
\end{proof}
The argument applies for any complete binary labeled tree $z$ of depth $n$ with input space $\mathcal{X}$, which means that the $0$-cover number $N(0,H,n) \leq \sum_{i=0}^{d} {n \choose i} k^i$. This completes the inductive argument and proves the theorem.
\end{proof}

\ifnum\neurips=0
\subsection{Lower Bound for $0$-Cover Function}\label{sec:0cov}
\else
\subsection{Lower bound for $0$-cover function}\label{sec:0cov}
\fi

To complement the upper bound given by our variant of Sauer's Lemma, we give a lower bound showing that the 0-cover function must grow exponentially in the $\Psi^B$-Littlestone dimension of a class.
\begin{lemma}\label{lemma:lowerb}
Let the $\Psi^B$-Littlestone dimension of hypothesis class $H$ be $d$. Then, \[N(0,H,d) \ge 2^d.\]
\end{lemma}
\begin{proof}
The case where $d = 0$ follows trivially from the convention that $N(0,H,0 ) = 1$. Hence, we consider $d>0$.

As a reminder, we note that the convention for input-labeled and output-labeled trees is that the root is at depth $1$, whereas the convention for $\Psi$-labeled trees is that the root is at depth $0$.

Since $H$ has $\Psi^B$-Littlestone dimension $d$, there is a complete $\Psi^B$-labeled binary tree $T$ of depth $d$ that is $\Psi^B$-shattered by $H$. Since $T$ is $\Psi^B$-shattered by $H$, for every root-to-leaf path $A$ of $T$, there is at least one function $f \in H$ that realizes that path. For each root-to-leaf path $A$, choose such a function and denote it by  $f_A$.

Construct an input labeled tree $T'$ of depth $d$ as follows: simply remove the unlabeled leaves of $T$, the labels on the edges and the collapsing map assigned to each node. Note that $T'$ has depth $d$ by the convention used that input-labeled trees have roots at depth $1$. Observe that the process of going from $T$ to $T'$ removes a layer of leaf nodes from $T$, and therefore each root-to-leaf path $B$ in $T'$ corresponds to two root-to-leaf paths, say $B_1$ and $B_2$, in $T$. Then, map functions $f_{B_1}$ and $f_{B_2}$ both to $B$. Hence, each root-to-leaf path of $T'$ has two functions mapped to it. See Figure~\ref{figure:psitoinput} for a visual depiction of this.
\begin{figure}[ht]
\centering
\includegraphics[width=0.9\linewidth]{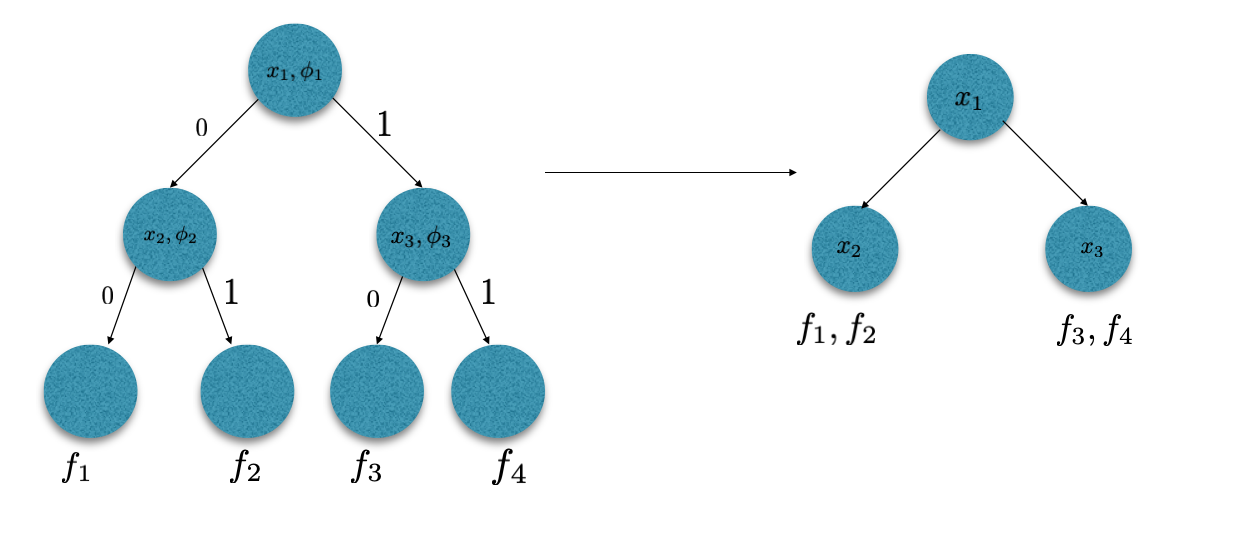}
\caption{The tree on the left represents the $\Psi^B$-labeled tree $T$ and the tree on the right represents the input-labeled tree $T'$ created from $T$. Functions $f_1, \dots, f_4$ are each assigned to a different root-to-leaf path in $T$ as shown, whereas in $T'$, two functions are assigned to each root-to-leaf path} \label{figure:psitoinput}
\centering
\end{figure}
We argue that any 0-cover of $T'$ must contain a distinct tree for each of the $2^d$ functions mapped to root-to-leaf paths of $T'$.

Fix any 2 root-to-leaf paths $A_1$ and $A_2$ in the complete $\Psi^B$-labeled tree $T$. Then there exists some node in $T$ where the two paths diverge. Let that node be labeled by example $x_{div}$ and collapsing map $\phi_{div}$. Then we have that $\phi_{div}\left(f_{A_1}(x_{div}) \right) \neq \phi_{div}\left(f_{A_2}(x_{div}) \right)$ since $f_{A_1}$ and $f_{A_2}$ realize $A_1$ and $A_2$ respectively. However, since $\phi_{div}$ is a function, this implies that $f_{A_1}(x_{div}) \neq f_{A_2}(x_{div})$. 

Now, consider the input-labeled tree $T'$. The node at which $A_1$ and $A_2$ diverge is in $T'$ as well. 

First, consider the case where it is a leaf node of $T'$. This means that both $f_{A_1}$ and $f_{A_2}$ are mapped to the same root-to-leaf path in $T'$. Hence, since  $f_{A_1}(x_{div}) \neq f_{A_2}(x_{div})$, we cannot construct a single output-labeled tree that covers both functions since then a single leaf node would need two labels.

On the other hand, if the node at which $A_1$ and $A_2$ diverge is not a leaf node of $T'$, then $f_{A_1}$ and $f_{A_2}$ are mapped to different root-to-leaf paths $B_1$ and $B_2$ in $T'$, which diverge at a node labeled by $x_{div}$. Hence, if we were trying to construct a single output-labeled tree $v$ such that the root-to-leaf paths $C_1$ and $C_2$ in $v$ corresponding to $B_1$ and $B_2$ in $T'$ had label sequences $s_1$ and $s_2$ such that $s_1 = f_{A_1}(B_1)$ and $s_2 = f_{A_2}(B_2)$, then at the node at which $B_1$ and $B_2$ diverge, we would need two different labels in $v$, which is impossible. 

This argument works for any pair $A_1$ and $A_2$ of root-to-leaf paths in $T$, and shows that the functions $f_{A_1}$ and $f_{A_2}$ require different trees in the $0$-cover of $T'$. The number of root-to-leaf paths in $T$ is $2^d$. Hence, we have that $N(0,H,d) \geq 2^d$. 
\end{proof}

\ifnum\neurips=0
\subsection{Putting the Pieces Together}
\else
\subsection{Putting the pieces together}
\fi
In this section, we prove Theorem~\ref{theorem:lsd} using the techniques we have built up.

Let the $\Psi^B$-Littlestone dimension of $H$ be $d_B$. The theorem is trivially true for $d_B = 0$, since by Lemmas~\ref{lem:olnatdim} and~\ref{lem:psilsdrel}, multiclass Littlestone dimension of $H$ (which is $d$) is a lower bound for $d_B$. 

Hence, let $d_B > 0$. By Lemmas~\ref{lem:olnatdim} and~\ref{lem:psilsdrel}, we have that $d \leq d_B$. Additionally, using Lemma~\ref{lemma:lowerb} and Theorem~\ref{thm:sauer} with $n = d_B \geq d$, we have that
\begin{equation}\label{eq:finalthm}
    2^{d_B} \leq N(0,H,d_B) \leq \left(\frac{ekd_B}{d}\right)^d.
\end{equation}

We will use the fact that for all positive real numbers $x,y$, $\ln x \leq xy - \ln(ey) $. Fix some constant $y < \ln 2$ to be chosen later. Removing the middle man from equatiion~\ref{eq:finalthm} and simplifying, we can write the following chain of inequalities.

\begin{align*}
    2^{d_B} \leq \left(\frac{ekd_B}{d}\right)^d 
    & \implies d_B \ln 2 \leq d\left(\ln\left(\frac{d_B}{d}\right) + \ln (ek) \right) \\
    & \implies d_B \ln 2 \leq d\left(y \cdot \frac{d_B}{d} - \ln(ey) + \ln (ek) \right) \\
    & \implies d_B(\ln 2 - y) \leq d\ln \left(\frac{k}{y}\right) \\
    & \implies d_B \leq \frac{1}{\ln2 - y} d\ln \left(\frac{k}{y}\right).
\end{align*}
Setting $y = \frac{1}{5} < \ln 2$, we get that
\[d_B \leq 6d \ln(k+1).\]
Next, by Lemmas~\ref{lem:lsbin} and~\ref{lem:psilsdrel}, we get that 
\[\max_{i=1, \dots, \log(k+1)} d_i \leq \Psi^{bin}_{LD}(H) \leq d_B \leq 6d \ln (k+1).\]
This proves the theorem.

\section{Tightness of Theorem~\ref{thm:online-reduction}}\label{sec:tightness}
In this section, we show that Theorem~\ref{thm:online-reduction} is tight up to constant factors.

\begin{theorem}
For all integers $k \geq 1, d \geq 0$, there exists a hypothesis class $H$ with label set $[k]$ such that the multiclass Littlestone dimension of $H$ is $d$ and the maximum Littlestone dimension over the binary restrictions of $H$ is at least $\frac{d}{10} \log_2(k+1)$. 
\end{theorem}
\begin{proof}
First, we will deal with trivial cases. When $d=0$, and $k \geq 1$, observe that any hypothesis class $H$ with $MLD(H)=d=0$ consists of a single function, which means each of the binary restrictions consists of a single function and the maximum Littlestone dimension over the binary restrictions is $0$. 

Next, when $k = 1$, and $d \geq 0$, any hypothesis class $H$ with multiclass Littlestone dimension $d$ has only one binary restriction, and $H|_1 = H$, which means that the multiclass Littlestone dimension of $H$ is equal to the Littlestone dimension of $H|_1$. For $k = 2,3,4$, $d \geq 0$ we can consider any binary hypothesis class $H$, and consider it as a multi-valued hypothesis class, with no functions in the class ever outputting the labels $2,3,4$. The multiclass Littlestone dimension in this case remains unchanged, and for $k=2,3$ the second binary restriction $H|_2$ is equal to $H$, while for $k=4$, the third binary restriction $H|_3$ is equal to $H$. Hence, the maximum Littlestone dimension over the binary restrictions also remains unchanged. Additionally $\log(k+1) \leq \log(5) < 10$, so the theorem is true in these cases as well.

Thus, it suffices to consider $k \geq 5$, $d \geq 1$. The proof strategy will start by proving the theorem for the case where $d=1$ and $k \geq 5$. We will then create an ``amplified'' class which will prove the theorem for arbitrary $d \geq 1$ and $k \geq 5$.  

Fix $k \geq 5$. Fix $k' = k$ if $k$ is even, and $k' = k+1$ if $k$ is odd. Consider an input domain $\mathcal{X} = [k'/2 - 1]$. Consider the class of threshold functions over the domain $\mathcal{X}$, where threshold functions $g_t$ are parametrized by a number $t \in \mathcal{X}$ and $g_t(x) = 1[t \geq x]$ indicates whether the parameter $t$ is at least the input $x$.

Then, the hypothesis class $F$ consists of functions parametrized by $t \in \mathcal{X}$, such that $f_t(x)=(g_t(x), t)$. Note that the number of labels in the label set of $F$ is $2(k'/2) = k'$, (the threshold function has binary outputs and $t \in [k'/2-1]$), and hence the output set can be encoded by $[k'-1]$. It is immediate that the multiclass Littlestone dimension of $F$ is $1$, since any label completely specifies the function $f_t$. Next, observe that the binary restriction $F|_1$ corresponds to the class of thresholds over $[k'/2 - 1]$, which is known to have Littlestone dimension lower bounded by $\log(k'/4)$ \cite{littlestone}. If $k$ is odd, then the number of labels in the label set of $F$ is $k+1$; however, if $k$ is even, then the number of labels is $k$. In order to make the number of labels $k+1$ so that the output set can be encoded by $[k]$, we add an extra label that is never used. (If $k$ is odd, we keep the hypothesis class $F$ unchanged). Note that adding unused labels does not change the multiclass Littlestone dimension of a class, nor does it change the maximum Littlestone dimension over its binary restrictions. 

Let the maximum Littlestone dimension over the binary resrictions of $F$ be $d_1$. Hence, we get that $d_1 \geq MLD(F) \log(k/4) \geq \frac{1}{10} MLD(F) \log(k+1)$ for $k \geq 5$. 

Next, we ``amplify'' the gap.
\begin{claim}[Gap amplification]\label{claim:gapamp}
Let $H$ be a hypothesis class with input space $\mathcal{X}$ and output space $[k]$ such that $MLD(H) = d$ and let the maximum Littlestone dimension over its binary restrictions be at least $D$. Then, for all $\ell \geq 1$, there exists a hypothesis class $H'$ with output space $[k]$ such that $MLD(H') = \ell d$ and the maximum Littlestone dimension over the binary restrictions of $H'$ is at least $\ell D$. 
\end{claim}
\begin{proof}
The hypothesis class $H'$ will have input space $\{1,\dots,\ell \} \times \mathcal{X}$ and output space $[k]$. It is constructed as follows:
\begin{align}
H' = \{h: \exists f_1, \dots, f_{\ell} \in H \text{ such that } \forall x \in \mathcal{X}, \forall j \in \{1,\dots,\ell\}, h((j,x)) = f_j(x) \}.
\end{align}
This class was used in Ghazi et al. \cite{closureld} in a different context.

First, we argue that the multiclass Littlestone dimension of $H'$ is upper bounded by $\ell D$. To see this, consider an online learner $A$ for $H'$ that runs $\ell$ copies of the multiclass Standard Optimal Algorithm \cite{daniely1}, $A_1, \dots, A_{\ell}$ for $H$ in parallel. Note that the multiclass Standard Optimal Algorithm is a deterministic online learning algorithm that makes at most $MLD(H)$ mistakes while online learning $H$. On getting an example $(j,x)$, $A$ outputs the prediction of $A_j$ on input $x$, and updates $A_j$ based on the feedback from the environment. It is immediate that the worst-case number of mistakes made by $A$ while online learning $H'$ is upper bounded by the sum over $j$ of the worst-case number of mistakes made by each $A_j$ while online learning $H$.

By work of Daniely et al. \cite{daniely1}, it is known that the multiclass Littlestone dimension of a hypothesis class lower bounds the worst-case number of mistakes of any deterministic online learner for that hypothesis class. 

Combining the arguments of the above two paragraphs, we get that 
\begin{equation}
MLD(H') \leq \ell \cdot MLD(H) = \ell d.
\end{equation}
We postpone the proof that $MLD(H') \geq \ell d$ till later since it will follow by an argument very similar to the argument below.

Next, we argue that the maximum Littlestone dimension over the binary restrictions of $H'$ is greater than or equal to $\ell D$. We will do this by relating it to the maximum Littlestone dimension over the binary restrictions of $H$. Let the maximum Littlestone dimension over the binary restrictions of $H$ be achieved by the $i^{th}$ binary restriction $H|_i$. Let its Littlestone dimension be $d_i \ge D$. 

Next, we make the following observation.
\begin{align}\label{eq:H'|_i}
H'|_i = \{h: \exists g_1, \dots, g_{\ell} \in H|_i \text{ such that } \forall x \in \mathcal{X}, \forall j \in \{1,\dots,\ell \}, h((j,x)) = g_j(x) \}.
\end{align}
We will represent any hypothesis $h \in H'|_i$ by $\ell$ hypotheses in $H|_i$. That is we will represent hypothesis $h$ by $(g_1, \dots, g_{\ell})$ where $g_1, \dots, g_{\ell}$ are $\ell$ hypotheses guaranteed by equation~\ref{eq:H'|_i}. 

We now argue that there exists a complete io-labeled binary tree of depth $\ell D$ with input set $\{1,\dots,\ell \} \times \mathcal{X}$ that is shattered by $H'|_i$. First, note that since $d_i$ is lower bounded by $D$, there exists a complete io-labeled binary tree $T$ of depth $D$ with input set $\mathcal{X}$ that is shattered by $H|_i$. We will derive a complete io-labeled tree $T'$ of depth $\ell D$ with input set $\{1,\dots,\ell \} \times \mathcal{X}$ from $T$.

First, create $\ell$ copies of $T$ and call them $T_1, \dots, T_{\ell}$. Remove the unlabeled leaves of $T_1, \dots, T_{\ell-1}$ (keep the unlabeled leaves of $T_{\ell}$) and change the node labelings of $T_j$ from $x$ to $(j,x)$. Keep the edge labels as is.

Next, define concatenation as follows. Let $T_1$ be the top-most tree. Create tree $T'$ by letting $T_2$ be both the left and right sub-tree to every leaf of $T_1$. In $T'$, label the two edges emanating out of each leaf of $T_1$ by $0$ and $1$ respectively. We will call $T'$ the \emph{concatenation} of $T_1$ and $T_2$, which we denote by $T' = T_1 \otimes T_2$.

Construct $T''$ as follows; let $T'' = T_1 \otimes (T_2 \otimes ( \dots \otimes T_\ell))$. Observe that $T''$ is a complete io-labeled binary tree of depth $\ell D$ with input set $\{1,\dots,\ell \} \times \mathcal{X}$.  

Consider root-to-leaf paths represented as a sequence of tuples of the form $(\text{node label}, \text{edge label})$. Then, any root-to-leaf path $B$ of $T''$ is of the following form. 
\[ B = \left(((1,x_1), edge_1), \dots, ((1,x_{\ell}), edge_{\ell}), ((2,x'_1), edge'_1), \dots, ((2,x'_{\ell}), edge'_{\ell}), \dots, ((\ell,x''_1), edge''_1), \dots, ((\ell,x''_{\ell}), edge''_{\ell})\right).\]

In this sense, $B$ is the concatenation of $\ell$ root-to-leaf paths $C_1, \dots, C_{\ell}$ of $T$ where 
\[C_1 = \left((x_1, edge_1), \dots, (x_{\ell}, edge_{\ell}) \right),\]
\[C_2 = \left((x'_1, edge'_1), \dots, (x'_{\ell}, edge'_{\ell}) \right),\]
and so on. 

Next, we argue that $T''$ is shattered by $H'|_i$. First, since $T$ is shattered by $H|_i$, every root-to-leaf path $C$ of $T$ is realized by a function $g_C \in H|_i$. Now, fix any root-to-leaf path $B$ of $T''$. By the argument described in the preceding paragraph, $B$ is the concatenation of $\ell$ root-to-leaf paths $C_1, \dots, C_{\ell}$ of $T$. The individual paths $C_1, \dots, C_{\ell}$ are realized by functions $g_{C_1}, \dots, g_{C_{\ell}} \in H|_i$ respectively. By construction,  the hypothesis $h = (g_{C_1}, \dots, g_{C_{\ell}}) \in H'|_i$ realizes $B$. This argument works for every root-to-leaf path $B \in T''$, and hence we have that $H'|_i$ shatters $T''$. This proves that the Littlestone dimension of $H'|_i$ is at least $\ell D$.

Finally, a similar argument can be used to show that the multiclass Littlestone dimension of $H$, i.e. $MLD(H)$ is at least $\ell d$. Hence, since the multiclass Littlestone dimension of $H'$ is at least $\ell d$ and at most $\ell d$, we have that it is exactly equal to $\ell d$.
\end{proof}

Now, applying the gap amplification Claim~\ref{claim:gapamp} to hypothesis class $F$ with multiclass Littlestone dimension $1$ and maximum Littlestone dimension over its binary restrictions at least $\frac{1}{10} \log(k+1)$, setting $\ell = d$, we get a hypothesis class $F'$. The maximum Littlestone dimension over the binary restrictions of $F'$ is at least $\frac{d}{10} \log(k+1)$ and the multiclass Littlestone dimension of $F'$ is $d$, proving the theorem.
\end{proof}

\section{Reverse Direction}\label{sec:revdirection}
\begin{theorem}\label{thm:lsdreverse}
Let $H$ by a hypothesis class with input set $\mathcal{X}$ and output set $[k]$. Let the multiclass Littlestone dimension of $H$ be $d$. Let $H|_1, H|_2, \dots, H|_{\log(k+1)}$ be the binary restrictions of $H$. Let the Littlestone dimensions of $H|_1, H|_2, \dots, H|_{\log(k+1)}$ be $d_1, \dots, d_{\log(k+1)}$. Then, 
\[d \leq \left[\max_{i=1, \dots, \log(k+1)} d_i \right]\log(k+1).\]
\end{theorem}
\begin{proof}
We prove the theorem by using the online learners for the binary restrictions $H|_i$ to construct an online learner for $H$. 

The online learner $A$ for $H$ works as follows. It runs $\log(k+1)$ online learners $A_1,\dots,A_{\log(k+1)}$ simultaneously, one for each of the binary restrictions $H|_i$. On getting an example $x$, it gets the $i^{th}$ online learner to predict the $i^{th}$ bit of the label by sending it $x$. It then concatenates these predictions and applies a binary to decimal conversion to obtain a prediction for the label in $[k]$. On receiving the true label $y$ from the environment, it updates the $i^{th}$ online learner with the $i^{th}$ bit of the binary expansion of the label.

If $A_1,\dots,A_{\log k}$ are all set to be the Standard Optimal Algorithm, then by the results of Littlestone \cite{littlestone}, we have that $A_i$ makes at most $d_i$ mistakes while online learning $H|_i$. Additionally, observe that $A$ makes a mistake if and only if at least one of the learners $A_1,\dots,A_{\log k}$ makes a mistake. Hence, the number of mistakes made by $A$ is upper-bounded by $\sum_{i=1}^{\log(k+1)} d_i \leq [\max_{i=1, \dots, \log(k+1)} d_i] \log(k+1).$ By work of Daniely et al. \cite{daniely1}, it is known that the multiclass Littlestone dimension of a hypothesis class lower bounds the worst-case number of mistakes of any deterministic online learner for that hypothesis class. Hence, we get that $d \leq  [\max_{i=1, \dots, \log(k+1)} d_i] \log(k+1)$. 
\end{proof}
Next, we describe a hypothesis class for which the above result is tight. The intuition is that such a hypothesis class is one where the information about the label is spread across the bits of the label as opposed to concentrated in a few bits.  

\begin{theorem}\label{thm:lsdreversetight}
For all integers $k \geq 1, d \geq 0$, there exists a hypothesis class $H$ with label set $[k]$ such that the multiclass Littlestone dimension of $H$ is $d \log(k+1)$ and the maximum Littlestone dimension over the binary restrictions of $H$ is $d$.
\end{theorem}
\begin{proof}
Let the binary restrictions $H|_1,\dots,H|_{\log(k+1)}$ all be the class of $d$-point functions over input space $\mathbb{N}$, that is for all functions $f_i \in H|_i$, there exist $d$ distinct natural numbers $x_1,\dots,x_d$ such that $f_i(x) = 1$ if and only if $x \in \{x_1,\dots,x_d\}$. Let $H$ be the hypothesis class with binary restrictions $H|_1,\dots,H|_{\log(k+1)}$ as defined above. 
It is argued in \cite{littlestone} that the class of $1$-point functions has Littlestone dimension $1$, and a straightforward extension of this argument shows that the Littlestone dimension of each of the binary restrictions $H|_1,\dots,H|_{\log(k+1)}$ is $d$. Hence, we are left to show that the multiclass Littlestone dimension of $H$ is $d \log(k+1)$.

First, an application of Theorem~\ref{thm:lsdreverse} shows that $MLD(H) \leq d \log(k+1)$. Hence, we are left to show that $MLD(H) \geq d \log(k+1)$. To prove this, we construct an io-labeled tree with label space $[k]$ of depth $d \log(k+1)$ that is shattered by $H$. 

To start, we observe that since $H|_i$ has multiclass Littlestone dimension $d$, there exists a tree $T_1$ of depth $d$ that is shattered by $H|_i$. Let the set of examples labeling nodes of this tree be $X_1$. Then, consider the subclass of $H|_2$ that corresponds to $d$-point functions over $\mathbb{N} \setminus X_1$ (that always predict $0$ on points from $X_1$). It is straightforward to see that this class has Littlestone dimension $d$, and so there exists a tree $T_2$ of depth $d$ that is shattered by $H|_2$, labeled only with examples from $\mathbb{N} \setminus X_1$. let the set of examples labeling nodes of $T_2$ be $X_2$. Then, consider the subclass of $H|_3$ that corresponds to $d$-point functions over $\mathbb{N} \setminus (X_1 \cup X_2)$. This class too has Littlestone dimension $d$, and so there exists a tree $T_3$ of depth $d$ that is shattered by $H|_3$, labeled only with examples from $\mathbb{N} \setminus (X_1 \cup X_2)$. We can define $T_4,\dots,T_{\log(k+1)}$ similarly. 

For all $j$, modify $T_j$, such that the edge labels of $T_j$ are $\log(k+1)$-bit binary strings as follows: for every edge label that is $1$, replace it by $(0,\dots,0,1,0,\dots,0)$ where the only $1$ is in the $j^{th}$ position. Similarly, replace all edge labels that are $0$ by the all-zero string. Remove the unlabeled leaf nodes of $T_1,\dots,T_{\log(k)}$ (keeping the leaf nodes of only a single tree $T_{\log(k+1)}$). Then, define tree $T$ of depth $d \log(k+1)$ as follows.
$$T = T_1 \otimes (T_2 \otimes (T_3 \otimes \dots)\dots)).$$
where the $\otimes$ operator is as defined in the proof of \Cref{claim:gapamp} in \Cref{sec:tightness}. Next, we argue that $T$ is shattered by $H$. Consider root-to-leaf paths represented as a sequence of tuples of the form $(\text{node label}, \text{edge label})$. Then, any root-to-leaf path $B$ of $T$ is of the following form. 
\[ B = \left((x_1,edge_1), \dots, (x_d,edge_d), (x'_1,edge'_1), \dots, (x'_d,edge'_d), \dots, (x''_1,edge''_1), \dots, (x''_d,edge''_d)\right).\]

In this sense, $B$ is the concatenation of $\log(k+1)$ root-to-leaf paths $C_1, \dots, C_{\log(k+1)}$ of $T_1,T_2,\dots,T_{\log(k+1)}$ respectively where 
\[C_1 = \left((x_1, edge_1), \dots, (x_d, edge_d) \right),\]
\[C_2 = \left((x'_1, edge'_1), \dots, (x'_d, edge'_d) \right),\]
and so on. 

Consider root-to-leaf path $C_1$ with the edge label $(1,0,\dots,0)$ replaced by $1$ and the all-zeros label replaced by $0$. Since $H|_1$ is the class of all $d$-point functions, there exists a function $f_1 \in H|_1$ that predicts $0$ on all inputs in the input sets $X_2,\dots,X_{\log(k+1)}$ that realizes $C_1$. By a similar argument, there exists a function $f_2 \in H|_2$ that predicts $0$ on all inputs in the input sets $X_1,X_3,X_4,\dots,X_{\log(k+1)}$ that realizes $C_2$ (with edge label $(0,1,0,\dots,0)$ replaced by $1$ and the all zeros label replaced by $0$). This argument can be extended to all root-to-leaf paths $C_j$ for $1 \leq j \leq \log(k+1)$. Hence, we have that the function $f = (f_1,\dots,f_{\log(k+1)}) \in H$ realizes root-to-leaf path $B$. Hence, $H$ shatters $T$, which implies that $MLD(H) \geq d \log(k+1)$, completing the proof.
\end{proof}
\section{Pure Differential Privacy}\label{sec:puredp}
In this section, we discuss multiclass private PAC learning under the constraint of $(\epsilon, 0)$-differential privacy. We first discuss how work done by Beimel et al. \cite{BNS19} for the binary case applies to this setting. Specifically, the \emph{probabilistic representation dimension} characterizes the sample complexity of pure private PAC learning in the multiclass setting upto logarithmic factors in the number of labels $k+1$. In the binary case, Feldman and Xiao \cite{feldman} showed that for any hypothesis class $H$, the representation dimension $RepDim(H)$ is asymptotically lower bounded by the Littlestone dimension of the class, that is, $RepDim(H) = \Omega(LDim(H))$. Their proof was via a beautiful connection to the notion of randomized one-way communication complexity. We will show the same result in the multiclass setting, through a (in our opinion, simpler) proof using the experts framework from online learning. 

First, we recall the notion of representation dimension, appropriately generalized to the multiclass setting.
\begin{definition}[Probabilistic Representation \cite{BNS19}]
Let $\mathcal{G}$ be a family of hypothesis classes $\{G_1,\dots,G_r\}$ with label set $[k]$ and let $P$ be a distribution over $\{ 1,\dots,r \}$. We say $(P,\mathcal{G})$ is an $(\alpha, \beta)$-probabilistic representation for a hypothesis class $H$ with label set $[k]$ and input set $\mathcal{X}$ if for every function $f \in H$, and every distribution $D$ over $\mathcal{X}$, 
\begin{equation}
    \Pr_P [\exists f_i \in G_i \text{ such that } \Pr_{x \sim D}[f(x) \neq f_i(x)] \leq \alpha ] \geq 1 - \beta.
\end{equation}
where the outer probability is over randomly choosing a hypothesis class $G_i \in \mathcal{G}$ according to $P$.
\end{definition}
\begin{definition}[Representation Dimension \cite{BNS19}]
Let $\mathcal{G}$ be a family of hypothesis classes $\{G_1,\dots,G_r\}$ with label set $[k]$. Let $size(\mathcal{G}) = \max_{G_i \in \mathcal{G}} \{ \ln |G_i| \}$. Then, the Representation Dimension of a hypothesis class $H$ with label set $[k]$ is defined as follows.
\begin{equation}
    RepDim(H) = \min \{size(\mathcal{G}): \exists P \text{ such that  $(P,\mathcal{G})$ is a $(1/4,1/8)$-probabilistic representation for $H$}\}.
\end{equation}
\end{definition}
Note that the constant $1/8$ chosen here (for the value of $\beta)$ is smaller than the constant $1/4$ chosen in the definition of representation dimension in Beimel et al. However, as they point out, $1/4$ is an arbitrary choice and their results are only changed by a constant factor by changing the constant. We choose $1/8$ since it simplifies a later argument concerning the connection between the representation dimension and multiclass Littlestone dimension.

Beimel et al. proved the following two lemmas, which continue to apply in the multiclass setting via the same proofs. 
\begin{lemma}[Lemma 14, \cite{BNS19}]\label{lem:pureprivuppbound}
If there exists a pair $(\mathcal{G},P)$ that $(\alpha, \beta)$-probabilistically represents a hypothesis class $H$ with label set $[k]$, then for every $\epsilon > 0$, there exists an $(\epsilon, 0)$-DP, $(6 \alpha, 4 \beta)$-accurate PAC learner $A$ for hypothesis class $H$ that has sample complexity $O(\frac{1}{\alpha \epsilon}(RepDim(H) + \ln (1/\beta)))$. 
\end{lemma}
\begin{lemma}[Lemma 15, \cite{BNS19}]
For any hypothesis class $H$, with label set $[k]$, if there exists an $(\epsilon, 0)$-DP, $(\alpha, 1/2)$-accurate PAC learner $A$ for $H$ that has sample complexity upper bounded by $m$, then there exists a $(\alpha, 1/4)$-probabilistic representation of $H$ such that $size(H) = O(m \epsilon)$. 
\end{lemma}
The latter lemma gives a lower bound of $\Omega(RepDim(H) / \epsilon)$ on the sample complexity of pure-private PAC learning in the multiclass setting. Note that this lower bound is independent of $\alpha$. Next, we we add in a dependence on $\alpha$ exactly as in Beimel et al. Unfortunately, this direct adaptation of the proof of Beimel et al. weakens the lower bound by an additive logarithmic term in the number of labels $k+1$.
\begin{lemma}
For any hypothesis class $H$, with label set $[k]$ and input set $\mathcal{X}$, if there exists an $(\epsilon, 0)$-DP, $(\alpha, 1/2)$-accurate PAC learner $A$ for $H$ that has sample complexity upper bounded by $m$, then there exists a $(1/4, 1/8)$-probabilistic representation of $H$ such that $size(H) = O(m \alpha \epsilon + \ln(k+1))$. 
\end{lemma}
\begin{proof}
Assume the existence of an $(\epsilon,0)$-DP, $(\alpha, 1/2)$-accurate PAC learner $A$ for $H$ with input space $\mathcal{X}$. We assume that $m \geq 3\ln(4) / 4 \alpha$ (this is without loss of generality since $A$ can ignore part of its sample).

Fix a target function $f \in H'$ and a distribution $D$ on input space $\mathcal{X}$. Let $0$ be an element of $\mathcal{X}$.

Define the following distribution $\tilde{D}$ on input space $\mathcal{X}'$.
\begin{equation}\label{eq:distrtilde}
\Pr_{\tilde{D}}[x] =
    \begin{cases}
    1-4\alpha + 4\alpha\Pr_D[x], & x=0 \\
    4 \alpha \Pr_D[x], & x \neq 0.
    \end{cases}
\end{equation}
Next, define $G_D^{\alpha} = \{\text{function } g: \Pr_{x \sim D}(g(x) \neq f(x)) \leq \alpha \}$. Since $A$ is  $(\alpha, 1/2)$-accurate, $$\Pr_{X \sim \tilde{D}^m, A} [A(X) \in  G_{\tilde{D}}^{\alpha}] \geq \frac{1}{2}. $$
In addition, by Equation~\ref{eq:distrtilde}, for every function $h: \mathcal{X}' \to [k]$ such that $\Pr_{x \sim D}[h(x) \neq f(x)] \geq \frac{1}{4}$,
\begin{equation}
\Pr_{x \sim \tilde{D}}[f(x) \neq h(x)] \geq 4 \alpha \Pr_{x \sim D}[f(x) \neq h(x)] \geq \alpha.
\end{equation}
Hence, $\Pr_{x \sim \tilde{D}}[f(x) \neq h(x)] \leq \alpha \implies \Pr_{x \sim D}[f(x) \neq h(x)] \leq \frac{1}{4},$ that is $G_{\tilde{D}}^{\alpha} \subseteq G_D^{1/4}.$ Therefore, $\Pr_{X \sim \tilde{D}^m, A} [A(X) \in  G_{D}^{1/4}] \geq \frac{1}{2}.$ We call a dataset of $m$ labeled examples \emph{good} if the unlabeled example $0$ appears at least $(1-8\alpha)m$ times in the dataset. Let $X$ be a dataset constructed by taking $m$ i.i.d samples from $\tilde{D}$ labeled by $f$. By a Chernoff bound, $X$ is good with probability at least $1-e^{-4 \alpha m/3}$. Hence, by a union bound,
$$\Pr_{A,\tilde{D}}[(A(X) \in G_D^{1/4}) \land (X \text{ is good})] \geq \frac{1}{2} - e^{-4 \alpha m/3} \geq \frac{1}{4}.$$
Therefore, there exists a dataset $X_{good}$ of $m$ examples (labeled by concept $f$) that is good such that $\Pr_{A}[A(X) \in G_D^{1/4}] \geq \frac{1}{4}$ where the probability is only over the randomness of the algorithm. For $\sigma \in [k]$, let $\vec{0}_{0}$ represent a dataset of size $m$ consisting only of the special element, such that every example is labeled by $\sigma$. Then, by group privacy, the fact that $X_{good}$ is good, there exists a $\sigma \in [k]$ such that
\begin{equation}
    \Pr_{A}[ A(\vec{0}_{\sigma}) \in G_{D}^{1/4}] \geq e^{-8\epsilon \alpha m}  \Pr_{A}[ A(X_{good}) \in G_{D}^{1/4}] \geq e^{-8\epsilon \alpha m} \cdot \frac{1}{4}.
\end{equation}
Consider a set $G$ consisting of the outcomes of $4 \ln(8) \cdot e^{8 \epsilon \alpha m}$ executions of $A(\vec{0}_0), A(\vec{0}_1),\dots,A(\vec{0}_k)$. The probability that $G$ does not contain a hypothesis $h \in G_D^{1/4}$, is then at most $(1 -  e^{-8\epsilon \alpha m} \cdot \frac{1}{4})^{4 \ln(8) \cdot e^{8 \epsilon \alpha m}} \leq \frac{1}{8}.$ Thus, if we let $\mathcal{G}$ be the set of all hypothesis classes of size at most $4(k+1) \ln(8) \cdot e^{8 \epsilon \alpha m}$ with label set $[k]$ and input set $\mathcal{X}'$, and set the distribution $P$ to be the distribution on $\mathcal{G}$ induced by $A(\vec{0}_0)$, then $(P,\mathcal{G})$ is a $(1/4,1/8)$-probabilistic representation of $H$, and $size(\mathcal{G}) =  \max_{G_i \in \mathcal{G}} \{ \ln |G_i| \} = \ln(4(k+1)\ln 8) + 8 m \epsilon \alpha = O(m \alpha \epsilon + \ln(k+1)).$
\end{proof}
The above lemma gives us that the sample complexity of any $(\epsilon,0)$-DP, $(\alpha, 1/2)$-accurate learning algorithm for a hypothesis class $H$ with label set $[k]$ is $m = \Omega \left(\frac{RepDim(H) - \ln(k+1)}{\alpha \epsilon} \right)$. Combined with Lemma~\ref{lem:pureprivuppbound}, this proves that the representation dimension captures the sample complexity of pure DP PAC learning in the multiclass setting up to logarithmic factors in $k+1$. 

Finally, we prove that the representation dimension of a hypothesis class $H$ is asymptotically lower bounded by the multiclass Littlestone dimension of $H$. To prove this, we will use the following lemmas from Daniely et al. \cite{daniely1}. 
\begin{lemma}[Theorem 5.1, \cite{daniely1}]\label{lem:worstcaseseq}
Let $H$ be a hypothesis class with label set $[k]$, such that $MLD(H) = d$. Then, for every online learning algorithm $A$ for $H$ (in the realizable setting), there exists a sequence of $d$ examples, such that $A$ makes at least $d/2$ mistakes in expectation on this sequence. 
\end{lemma}
Our proof makes use of the experts framework in online learning. In this framework, at each time step $t$, before the online learner chooses its prediction, it gets $N$ pieces of advice from experts with opinions about what the correct prediction should be. There is a classical algorithm called the Weighted Majority Algorithm \cite{littlestonewar,daniely1} that achieves the following guarantee in this framework for every sequence of length $T$. 

\begin{lemma}[Page 20, \cite{daniely1}]\label{thm:weightedmaj}
Consider $N$ experts. Let $A$ be the Weighted Majority Algorithm. Then, for all sequences $S$ of length $T$, labeled by an unknown hypothesis in a hypothesis class $H$ with label set $[k]$, if the number of mistakes made by the $i^{th}$ expert on $S$ is $L_{i,T}$  and the number of mistakes made by $A$ on $S$ is $L_{A,T}$, then,
\begin{equation}
 \mathbb{E}_A[ L_{A,T} - \min_{i \in [N]} \{L_{i,T}\}] \leq \sqrt{\frac{1}{2} \ln(N) T}. 
\end{equation}
\end{lemma}
Next, we prove the main result.
\begin{theorem}
For all $k \in \mathbb{N}$, for any hypothesis class $H$ with label set $[k]$, $RepDim(H) = \Omega(MLD(H))$.
\end{theorem}
\begin{proof}
We can assume $MLD(H) > 0$, since when $MLD(H)=0$, the result is vacuously true. By the definition of representation dimension, there exists a $(1/4,1/8)$-probabilistic representation $(P,\mathcal{G})$ for hypothesis class $H$ where $size(\mathcal{G}) = RepDim(H)$.

Now, consider the following online learner $A$ for $H$. It first samples a hypothesis class $G_i \in \mathcal{G}$ from $P$. It then runs the Weighted Majority Algorithm with the set of experts being the functions of hypothesis class $G_i$ in order to make predictions at every timestep.

Consider the worst case sequence $S$ for $A$ of length $T = MLD(H)$ guaranteed by Lemma~\ref{lem:worstcaseseq}. Let $D$ be the empirical distribution corresponding to this sequence. Fix an unknown hypothesis $f \in H$ labeling this sequence.  By the definition of probabilistic representation, if a hypothesis class $G_i$ is sampled from $P$, with probability at least $7/8$, there exists a hypothesis $g \in G_i$ such that $\Pr_{x \sim D}[g(x) \neq f(x)] \leq 1/4$. This implies that the number of mistakes that $g$ makes on $S$ labeled by $f$ is at most $MLD(H)/4$. 

Let $L_A$ be a random variable denoting the number of mistakes that $A$ makes on $S$. Define the event $E$ as follows: "Class $G_i$ sampled from $P$ is such that for all distributions $D'$ and for all functions $f' \in H$, there exists a function $g \in G_i$ such that  $\Pr_{x \sim D'}[g(x) \neq f'(x)] \leq 1/4$." Then, $\Pr[E] \geq 7/8$. Using the law of total expectation, and the fact that the maximum number of mistakes is the length of the sequence $S$, we can then write that
\begin{align}
    \mathbb{E}_A[L_A] & = \mathbb{E}_A[L_A \mid E]\Pr[E] + \mathbb{E}_A[L_A \mid \overline{E}]\Pr[\overline{E}] \\
    & \leq  \mathbb{E}_A[L_A \mid E] + \frac{MLD(H)}{8}.
\end{align}
Finally, observe that conditioned on event $E$, there is an expert in $G_i$ that makes at most $MLD(H)/4$ mistakes on $S$. Hence, applying Theorem~\ref{thm:weightedmaj}, we get that  $$\mathbb{E}_A[L_A \mid E] \leq  MLD(H)/4 + \sqrt{\frac{1}{2} RepDim(H) MLD(H)},$$
where we have used the fact that the natural logarithm of the number of experts in $G_i$ is at most the representation dimension of hypothesis class $H$. Hence, putting together the arguments from the previous two paragraphs, we have that
$$\mathbb{E}_A[L_A] \leq  3MLD(H)/8 + \sqrt{\frac{1}{2} RepDim(H) MLD(H)}.$$
Next, by Lemma~\ref{lem:worstcaseseq}, we have that  $\mathbb{E}_A[L_A] \geq MLD(H)/2$. Combining the above two equations, we get that
$$\frac{MLD(H)}{2} \leq 3MLD(H)/8 + \sqrt{\frac{1}{2} RepDim(H) MLD(H)} \implies \frac{MLD(H)}{32} \leq RepDim(H).$$
\end{proof}
\newpage
\bibliographystyle{alpha}

\bibliography{bibl}
\newpage
\appendix
\section{Boosting Private Learners}\label{sec:Boosting}

The following result shows that differentially private learners with constant accuracy and confidence parameters can be generically boosted to make these parameters arbitrarily small, with only a mild dependence on these parameters.

\begin{theorem}\label{thm:boosting}
Let $\epsilon \le 1/4$ and let $H$ be a binary concept class that has an $(\epsilon, \delta)$-differentially private and $(\alpha = 1/4, \beta = 1/4)$-accurate PAC learner using $SC_{1/4, 1/4}$ samples. Then for every $\alpha, \beta > 0$, there is an $(\epsilon, \delta)$-differentially private learner for $H$ using
\[O\left(\frac{SC_{1/4, 1/4} \cdot \log(\log(1/\alpha)/\beta) \cdot \log(1/\alpha)}{\alpha}\right).\]
samples.
\end{theorem}

\begin{proof}[Proof sketch]
We apply the framework of private boosting via lazy Bregman projections described in Bun et al.~\cite{bcs20}. The hypotheses of the theorem guarantee the existence of a ``weak'', i.e., $(1/4, 1/4)$-accurate, learner $A$ for the class $H$. Boosting repeatedly runs (a modification of) $A$ on a sequence $\mu_1, \dots \mu_T$ of reweightings of the input sample, and aggregates the resulting hypotheses $h_1, \dots, h_T$ into a new hypothesis $h^*$ with arbitrarily small accuracy parameters $\alpha, \beta$. 

First, we explain how to boost the confidence parameter of this learner from $1/3$ to an arbitrary $\beta > 0$. We do this by repeating the algorithm $A$ some $k = O(\log(1/\beta))$ times, producing a sequence of candidate hypotheses $h_1, \dots, h_k$. We then use the \emph{exponential mechanism}~\cite{MT07} to identify an $h_i$ with approximately minimal error with respect to a fresh sample of size $O(\log(k/\beta) / \epsilon)$. Overall this results in a $(2\epsilon, \delta)$-differentially private and $(1/3, \beta)$-accurate PAC learner $A'$ using $O(SC_{1/4, 1/4} \cdot \log(1/\beta))$ samples.

We now use $A'$ to construct a weak learner satisfying the conditions needed to apply the framework in Bun et al.~\cite{bcs20}. Namely, the weak learner must be able to generate accurate hypotheses with respect to any distribution over the input sample, as well as a generalized form of differential privacy that ensures similar output distributions given neighboring datasets \emph{and} reweightings. Below, let $\Delta_n$ denote the set of probability distributions over $\{1, \dots, n\}$. Moreover, we say that a distribution $\mu \in \Delta_n$ is $s$-smooth if $\mu(i) \le s$ for every $i \in \{1, \dots, n\}$.

\begin{claim} \label{claim:wkl}
For parameters $n, T \in \mathbb{N}$ and $s \in (0, 1)$, there exists a randomized algorithm $W : \mathcal{X}^n \times \Delta_n \to \{0, 1\}^{\mathcal{X}}$ with the following properties:
\begin{itemize}
    \item Accuracy: If $1 / s \ge O(SC_{1/4, 1/4} \cdot \log(T/\beta))$, then for every sample $D = ((x_1, y_1), \dots, (x_n, y_n))$ and distribution $\mu \in \Delta_n$, with probability at least $1 - \beta / 2T$, the weak learner outputs a hypothesis $h$ such that
    \[\sum_{i = 1}^n |h(x_i) - y_i| \mu(i) \le \frac{1}{3}. \]
    
    \item Privacy: For every pair of neighboring samples $D, D'$ and $s$-smooth distributions $\mu, \mu'$ at statistical distance at most $s$, we have that $W(D, \mu)$ and $W(D', \mu')$ are $(2\epsilon, 3\delta)$-indistinguishable.
\end{itemize}
\end{claim}

We obtain the algorithm $W$ from $A'$ as follows. On input a training set $D$ and distribution $\mu$, sample $m = 1/16 s$ examples without replacement from $D$ according to $\mu$ and run $A'$ on the result. By the accuracy guarantee of $A'$, this meets the desired accuracy condition as long as $m \ge O(SC_{1/4, 1/4} \cdot \log(T/\beta))$. Moreover, by a privacy amplification by subsampling argument~\cite[Lemma 6.5]{drv10}, we have that $W(D, \mu)$ and $W(D', \mu')$ are $(32\epsilon s m, \delta(1+e^{8\epsilon s m}))$-indistinguishable for every $D \sim D'$ and $s$-smooth $\mu, \mu'$ at statistical distance $s$. Our choice of $m = 1/16s$ and $\epsilon \le 1$ prove the claim.

The lazy Bregman boosting procedure of Bun et al.~\cite{bcs20} shows that after $T = O(\log(1/\alpha))$ rounds of boosting, the aggregated hypothesis has training error at most $\alpha / 2$ except with probability at most $\beta / 2$. Moreover, for every pair of neighboring inputs $D, D'$, the sequences of distributions $\mu_1, \dots, \mu_T$ and $\mu_1', \dots, \mu_T'$ constructed over the course of boosting are all $(1/\alpha n)$-smooth and element-wise at statistical distance at most $1/\alpha n$. The accuracy condition of Claim~\ref{claim:wkl} kicks in as long as $1/s = \alpha n \ge O(SC_{1/4, 1/4} \cdot \log(T/\beta))$, i.e., if $n \ge O\left(\frac{1}{\alpha} \cdot SC_{1/4, 1/4} \cdot \log(T/\beta)\right)$. Meanwhile, by the privacy condition of Claim~\ref{claim:wkl}, each individual round of boosting is $(2\epsilon, 3\delta)$-differentially private, so by the basic composition theorem for differential privacy, the algorithm as a whole is $(2T\epsilon, 3T\delta)$-differentially private.

Finally, we apply a (different) privacy amplification by subsampling argument~\cite[Lemma 4.12]{bnsv} one more time to convert this algorithm from a $(2T\epsilon, 3T\delta)$-differentially private algorithm with sample complexity $O\left(\frac{1}{\alpha} \cdot SC_{1/4, 1/4} \cdot \log(T/\beta)\right)$ into a $(\epsilon, \delta)$-differentially private algorithm with sample complexity
\[O\left(\frac{SC_{1/4, 1/4} \cdot \log(T/\beta) \cdot T}{\alpha}\right) = O\left(\frac{SC_{1/4, 1/4} \cdot \log(\log(1/\alpha)/\beta) \cdot \log(1/\alpha)}{\alpha}\right).\]
By standard generalization bounds for approximate empircal risk minimization~\cite{BlumerEHW89}, achieving low error on a training set of size $O(\frac{1}{\alpha} \cdot (VC(H) + \log(1/\alpha\beta)))$ suffices to achieve low error on the population, where $VC$ denotes the Vapnik-Chervonenkis dimension. This is already true using the number of samples given above, using the fact that the private learner we started with is a PAC learner and hence $SC_{1/4, 1/4} \ge \Omega(VC(H))$.
\end{proof}
\begin{corollary}[Boosted Learner from \cite{ghazi}]\label{thm:ghazi-2}
Let $G$ be any binary hypothesis class with Littlestone dimension $d_L \geq 0$. Then, for any $\epsilon \in [0,1/4], \delta, \alpha, \beta \in [0,1]$,
for some \[n = O\left(\frac{d_L^6 \log^4(\frac{d_L}{\alpha \beta \epsilon \delta})}{\epsilon \alpha}\right),\]
there is an $(\epsilon, \delta)$-differentially private, $(\alpha, \beta)$-accurate PAC learning algorithm $B$ for $G$ with sample complexity upper bounded by $n$.
\end{corollary}
\begin{proof}
Using the binary PAC learner from Ghazi et al. \cite{ghazi}, by the sample complexity bound given in theorem~\ref{thm:ghazi}, for $\epsilon \in [0,1/4], \delta \in [0,1]$ we get an $(\epsilon, \delta)$-differentially private $(1/4, 1/4)$-accurate PAC learner with sample complexity upper bounded by $O\left(\frac{d_L^6 \log^2(\frac{d_L}{\epsilon \delta})}{\epsilon}\right)$. For the case where $d_L =0$, this directly proves the corollary. So assume $d_L > 0$. Applying the boosting procedure described in Theorem~\ref{thm:boosting}, we get an $(\epsilon, \delta)$-differentially private $(\alpha, \beta)$-accurate PAC learner with sample complexity upper bounded
by $O\left(\frac{d_L^6 \log^2(\frac{d_L}{\epsilon \delta})\cdot \log(\log(1/\alpha)/\beta) \cdot \log(1/\alpha)}{\epsilon \alpha} \right) = O\left(\frac{d_L^6 \log^4(\frac{d_L}{\epsilon \delta \alpha \beta})}{\epsilon \alpha} \right)$. 
\end{proof}


\end{document}